\documentclass{article}

\usepackage{microtype}
\usepackage{graphicx}
\usepackage{subcaption}
\usepackage{booktabs}

\usepackage{hyperref}

\usepackage[preprint]{icml2026}

\usepackage{amsmath}
\usepackage{amssymb}
\usepackage{amsthm}
\usepackage{amsfonts}
\usepackage{nicefrac}
\usepackage{bm}

\usepackage{xcolor}
\usepackage{algorithm}
\usepackage{algorithmic}
\usepackage{multirow}
\usepackage[capitalize,noabbrev]{cleveref}
\usepackage{enumitem}
\usepackage{tcolorbox}
\usepackage{tikz}
\usepackage{pgfplots}
\pgfplotsset{compat=1.17}

\theoremstyle{plain}
\newtheorem{theorem}{Theorem}[section]
\newtheorem{lemma}[theorem]{Lemma}
\newtheorem{proposition}[theorem]{Proposition}

\theoremstyle{definition}
\newtheorem{definition}[theorem]{Definition}
\newtheorem{remark}[theorem]{Remark}

\newtheorem{fact}[theorem]{Fact}

\DeclareMathOperator*{\argmin}{arg\,min}

\newcommand{\R}{\mathbb{R}}

\newcommand{\cT}{\mathcal{T}}
\newcommand{\cS}{\mathcal{S}}
\newcommand{\bx}{\bm{x}}
\newcommand{\by}{\bm{y}}
\newcommand{\bg}{\bm{g}}
\newcommand{\bA}{\bm{A}}
\newcommand{\bB}{\bm{B}}
\newcommand{\bG}{\bm{G}}
\newcommand{\bV}{\bm{V}}
\newcommand{\bd}{\bm{d}}
\newcommand{\bD}{\bm{D}}
\newcommand{\bY}{\bm{Y}}
\newcommand{\bL}{\bm{L}}
\newcommand{\bP}{\bm{P}}
\icmltitlerunning{NOVA: Numerical Optimization of Vandermonde Arithmetic}

\begin{document}

\twocolumn[
\icmltitle{NOVA: Discovering Well-Conditioned Winograd Transforms \\
through Numerical Optimization of Vandermonde Arithmetic}

% Non-anonymous author information for arXiv
\begin{icmlauthorlist}
\icmlauthor{Jayant Lohia}{ind}
\end{icmlauthorlist}

\icmlaffiliation{ind}{Independent Researcher}

\icmlcorrespondingauthor{Jayant Lohia}{jayantlohia16@gmail.com}

\icmlkeywords{NOVA, Winograd convolution, numerical stability, condition number, evolution strategy, quantization, deep learning inference}

\vskip 0.3in
]

\printAffiliationsAndNotice{}

%==============================================================================
% ABSTRACT
%==============================================================================
\begin{abstract}
% Abstract - arXiv NOVA version

Winograd convolution is the standard algorithm for efficient inference, reducing arithmetic complexity by 2.25$\times$ for 3$\times$3 kernels. However, it faces a critical barrier in the modern era of low-precision computing: numerical instability. As tiles scale to maximize efficiency (e.g., F(6,3), F(8,3)), the condition numbers of standard integer-based transforms explode---reaching $\kappa \approx 2 \times 10^5$ for F(8,3)---rendering them unusable in FP16 or Int8.

We introduce \textbf{NOVA} (\textbf{N}umerical \textbf{O}ptimization of \textbf{V}andermonde \textbf{A}rithmetic), a discovery framework that breaks the decades-old convention of integer interpolation. Treating Winograd point selection as a continuous optimization problem, NOVA searches the manifold $\mathbb{R}^{n-1}$ via Evolution Strategy, snaps candidates to simple rationals, and guarantees correctness via symbolic verification. This process uncovers a hidden landscape of stable, fractional configurations---such as $\{\pm\nicefrac{5}{6}, \pm\nicefrac{7}{6}, \pm\nicefrac{3}{5}\}$---that defy traditional vocabulary constraints.

The impact is transformative: NOVA improves the conditioning of F(8,3) by \textbf{415$\times$} in 1D, which squares to a \textbf{172,484$\times$} improvement for 2D convolution. In real-world FP16 ImageNet inference, where standard transforms collapse to random chance (e.g., 4.7\% accuracy on VGG-16), NOVA's points restore full accuracy (75--78\%), recovering over \textbf{70 percentage points} without retraining, calibration, or learned parameters. These discovered transforms act as drop-in replacements, effectively unlocking the efficiency of large-tile Winograd convolution for next-generation hardware.

\end{abstract}

%==============================================================================
% INTRODUCTION
%==============================================================================
\section{Introduction}
\label{sec:intro}
% Introduction - Open Discovery of Winograd Interpolation Points
% CORRECTED: ES-based discovery in continuous space, not vocabulary-based RL

Winograd's minimal filtering algorithm~\cite{winograd1980arithmetic} is the theoretical engine behind efficient convolution, reducing the arithmetic cost of 3$\times$3 kernels by 2.25$\times$. This efficiency has driven its widespread adoption in deep learning libraries~\cite{lavin2015fast}. However, a critical \textbf{hardware-algorithm gap} has emerged: while modern inference hardware (NPUs, mobile GPUs) relies on low-precision arithmetic (FP16, Int8) for speed and energy efficiency, Winograd's mathematical formulation assumes infinite precision.

This tension creates a hidden tax on efficiency. To maintain numerical stability, deployments are forced to use small, less efficient tiles (e.g., F(4,3)), sacrificing the theoretical gains of larger tiles like F(6,3) or F(8,3). The culprit is the ill-conditioning of the Vandermonde matrices used in the transform~\cite{gautschi1975optimally}. Standard integer points $\{0, \pm 1, \pm 2, \ldots\}$ produce condition numbers that explode exponentially (Figure~\ref{fig:conditioning}): $\kappa_2$ jumps from 42.5 for F(4,3) to nearly 200,000 for F(8,3). In the unforgiving regime of FP16/Int8, this instability turns signal into noise, causing catastrophic accuracy collapse~\cite{barabasz2019error}.

% Conditioning Improvement Figure
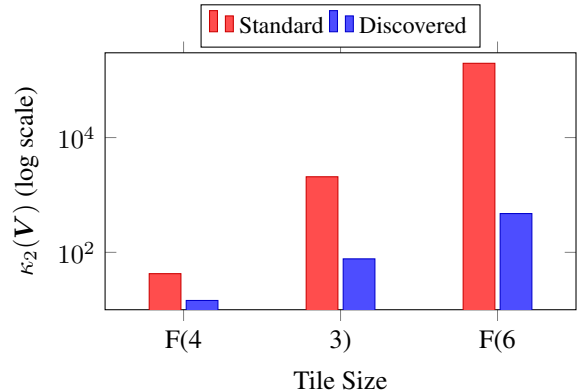
\begin{figure}[t]
\centering
\begin{tikzpicture}
\begin{axis}[
    ybar,
    bar width=12pt,
    width=0.95\columnwidth,
    height=5cm,
    ylabel={$\kappa_2(\bV)$ (log scale)},
    xlabel={Tile Size},
    ymode=log,
    symbolic x coords={F43, F63, F83},
    xtick=data,
    xticklabels={F(4,3), F(6,3), F(8,3)},
    legend style={at={(0.5,1.02)}, anchor=south, legend columns=2, font=\small},
    ymin=10, ymax=300000,
    enlarge x limits=0.25,
]
\addplot[fill=red!70, draw=red!80!black] coordinates {
    (F43, 42.5)
    (F63, 2075)
    (F83, 196900)
};
\addplot[fill=blue!70, draw=blue!80!black] coordinates {
    (F43, 14.5)
    (F63, 77)
    (F83, 474)
};
\legend{Standard, Discovered}
\end{axis}
\end{tikzpicture}

\vspace{-5pt}
\caption{Condition number $\kappa_2(\bV)$ comparison. Standard integer points exhibit exponential growth. Discovered fractional points achieve 2.9$\times$--415$\times$ improvement by clustering points closer together than integers.}
\label{fig:conditioning}
\end{figure}

\paragraph{The Discovery Question.}
For decades, the design of Winograd transforms has been constrained by a historical artifact: the preference for simple integer interpolation points. Prior work has hinted that this design space is suboptimal~\cite{alam2022winograd}, showing that restricted families of rationals can reduce error. But these approaches merely scratched the surface.

This paper asks: \emph{If we strip away the historical constraint of "integer-only" points, can we discover a new class of transforms that are mathematically precise yet numerically robust enough for modern low-precision hardware?}

\paragraph{Open Discovery via Evolution Strategy.}
We present a method for \textbf{open discovery} of well-conditioned Winograd interpolation points:
\begin{enumerate}
    \item \textbf{Evolution Strategy (ES)} searches the continuous space $\R^{n-1}$ (with the point at infinity always included)
    \item \textbf{Snap-to-rational} converts continuous solutions to simple fractions (e.g., denominator $\leq 6$)
    \item \textbf{Symbolic verification} via SymPy ensures exact mathematical correctness
\end{enumerate}

This approach differs fundamentally from vocabulary-based methods: we search the full continuous space and discover configurations that lie \emph{outside any predefined vocabulary}.

\paragraph{Key Discovery: Points Outside Any Vocabulary.}
Our ES discovers fractional points like $\{\pm\nicefrac{5}{6}, \pm\nicefrac{7}{6}, \pm\nicefrac{3}{5}\}$ that are not present in typical vocabularies. These points dramatically improve conditioning:

\begin{center}
\small
\setlength{\tabcolsep}{2pt}
\begin{tabular}{lcccc}
\toprule
Tile & Std $\kappa_2$ & Disc $\kappa_2$ & Impr. & Discovered Points (+$\infty$) \\
\midrule
F(4,3) & 42.5 & 14.5 & \textbf{2.9$\times$} & $\{0, \pm\frac{5}{6}, \pm\frac{7}{6}\}$ \\
F(6,3) & 2,075 & 77 & \textbf{27$\times$} & $\{0, \pm\frac{3}{5}, \pm 1, \pm\frac{7}{6}\}$ \\
F(8,3) & 196.9k & 474 & \textbf{415$\times$} & $\{0, \pm\frac{2}{5}, \pm\frac{5}{6}, \pm 1, \pm\frac{7}{6}\}$ \\
\bottomrule
\end{tabular}
\end{center}

The key observation is that optimal points cluster closer together than integers. Replacing large integers with fractions like $\pm\nicefrac{5}{6}$, $\pm\nicefrac{7}{6}$, $\pm\nicefrac{3}{5}$ keeps the Vandermonde entries smaller, dramatically reducing conditioning.

\paragraph{Kronecker Products: 1D Improvements Square in 2D.}
For 2D convolution, transformation matrices are Kronecker products: $\bA_{2D} = \bA \otimes \bA$. Since $\kappa_2(\bA \otimes \bA) = \kappa_2(\bA)^2$, 1D improvements \emph{square} in 2D: our 27$\times$ improvement for F(6,3) becomes \textbf{733$\times$} for F(6$\times$6,3$\times$3), and F(8$\times$8,3$\times$3) achieves \textbf{172,484$\times$} improvement.

\paragraph{The Float16 Trade-off.}
We identify a nuanced finding for low-precision deployment: discovered fractional points have \emph{better} conditioning but \emph{worse} float16 representation error than integers. Standard points $\{0, \pm 1, \pm 2\}$ are exactly representable in float16, while fractions like $\nicefrac{5}{6}$ incur quantization error in the transformation matrices themselves. This motivates dtype-aware discovery that constrains search to exactly representable rationals.

\paragraph{Contributions.}

\begin{enumerate}
    \item \textbf{Problem Reframing}: We show that Winograd point selection is a \emph{continuous optimization problem} over $\R^{n-1}$, not a vocabulary or combinatorial search. Prior parameterized searches~\cite{alam2022winograd} explore constrained subsets; our ES + snap-to-rational + symbolic verification pipeline searches the full space, discovering points like $\{\pm\nicefrac{5}{6}, \pm\nicefrac{7}{6}\}$ outside any predefined vocabulary.

    \item \textbf{Conditioning as Causal Bottleneck}: We establish that $\kappa(\bV)$ is not merely correlated with error but \emph{explains} FP16/INT8 collapse. Improvements propagate: $\kappa(\bV) \to \kappa(\bA,\bB,\bG) \to$ norm products $\to$ network accuracy. Discovered points achieve \textbf{2.9--415$\times$} $\kappa$ improvement (up to \textbf{172,484$\times$} in 2D), with recovery tracking $\sqrt{\kappa}$ ratios. Optimal points cluster more tightly than integers, approximating Chebyshev-like spacing under rational constraints. We also identify and resolve a conditioning--representability trade-off via dtype-aware discovery.

    \item \textbf{Symbolic Verification}: SymPy-based exact rational arithmetic verifies decomposition correctness to machine precision, eliminating floating-point verification errors that plagued prior empirical searches.

    \item \textbf{Large-Scale Validation}: On ImageNetV2 (\textbf{30,000 images}) across 6 architectures, standard FP16 F(6,3) collapses to 4.7--10.8\%; discovered points recover \textbf{75--81\%}---\textbf{67--73pp recovered} without retraining. We outperform Alam et al.\ by 21\% in $\kappa$ and match their network accuracy.

    \item \textbf{Drop-in Deployment}: Unlike methods requiring retraining~\cite{liu2020winograd}, basis changes~\cite{barabasz2020legendre}, or learned scales~\cite{chikin2022tapwise}, our points are drop-in replacements preserving arithmetic count, pipeline, pretrained weights, and latency. We demonstrate a direct path to stable FP16 F(6,3) inference without calibration or learned parameters.
\end{enumerate}

\paragraph{Significance.}
Our work establishes that the standard practice of using integer interpolation points is dramatically suboptimal. By exploring the full continuous space and snapping to simple rationals, we find configurations that improve conditioning by orders of magnitude. The discovered configurations can be directly used in existing Winograd implementations---only the interpolation points and derived matrices need to change.

%==============================================================================
% BACKGROUND
%==============================================================================
\section{Background}
\label{sec:background}
% Background Section - Open Discovery
% CORRECTED: Focus on Cook-Toom and search space structure

We review the Cook-Toom algorithm for fast convolution, the Vandermonde conditioning problem, and the structure of the search space that enables open discovery.

\subsection{Winograd Convolution via Cook-Toom}
\label{sec:cook-toom}

Consider a 1D convolution of an input signal $\bx \in \R^m$ with a kernel $\bg \in \R^r$, producing output $\by \in \R^m$ where $y_i = \sum_{k=0}^{r-1} g_k \cdot x_{i+k}$. Direct computation requires $m \cdot r$ multiplications.

The Cook-Toom algorithm~\cite{toom1963complexity} reduces this to $n = m + r - 1$ multiplications by interpreting convolution as polynomial multiplication. Let $x(t) = \sum_{i=0}^{m-1} x_i t^i$ and $g(t) = \sum_{k=0}^{r-1} g_k t^k$. The product $y(t) = x(t) \cdot g(t)$ has degree $m + r - 2$, so $n$ evaluations at distinct points suffice to recover $y(t)$ via interpolation.

\paragraph{Algorithm Structure.}
Choose $n$ distinct interpolation points $\cS = \{\alpha_0, \alpha_1, \ldots, \alpha_{n-2}, \infty\}$. The algorithm proceeds:
\begin{enumerate}
    \item \textbf{Input Transform}: $\tilde{\bx} = \bB^T \bx$ (evaluate input polynomial at each point)
    \item \textbf{Kernel Transform}: $\tilde{\bg} = \bG \bg$ (evaluate kernel polynomial at each point)
    \item \textbf{Element-wise Multiply}: $\tilde{\by} = \tilde{\bx} \odot \tilde{\bg}$ (only $n$ multiplications)
    \item \textbf{Output Transform}: $\by = \bA^T \tilde{\by}$ (interpolate to recover output)
\end{enumerate}

The matrices $\bA, \bB, \bG \in \R^{n \times *}$ are derived from Lagrange interpolation systems based on $\cS$.

\paragraph{The Point at Infinity.}
To achieve the minimal $n = m + r - 1$ multiplications, one interpolation point is placed at infinity ($\alpha_{n-1} = \infty$). This corresponds to extracting the leading coefficient of the product polynomial. We denote the $n-1$ finite points as $\cS_{\text{fin}} = \{\alpha_0, \ldots, \alpha_{n-2}\}$.

\textbf{Scope}: Our formulation follows the standard Cook-Toom/Winograd construction~\cite{lavin2015fast,winograd1980arithmetic} where the infinity point enables minimal multiplication count. Alternative formulations exist: (1) modular/CRT-based methods~\cite{yepremyan2020rns} avoid infinity via exact arithmetic; (2) some implementations use $n$ finite points with one additional multiplication. Our discovered points apply to the standard formulation used in cuDNN, TensorRT, and most deep learning frameworks.

\subsection{Vandermonde Conditioning}
\label{sec:conditioning}

The transformation matrices are fundamentally tied to Vandermonde matrices. For the $n-1$ \emph{finite} points $\cS_{\text{fin}} = \{\alpha_0, \ldots, \alpha_{n-2}\}$:
\begin{equation}
    \bV(\cS_{\text{fin}}) = \begin{bmatrix}
        1 & \alpha_0 & \cdots & \alpha_0^{n-2} \\
        1 & \alpha_1 & \cdots & \alpha_1^{n-2} \\
        \vdots & \vdots & \ddots & \vdots \\
        1 & \alpha_{n-2} & \cdots & \alpha_{n-2}^{n-2}
    \end{bmatrix} \in \R^{(n-1) \times (n-1)}
\end{equation}

\paragraph{Handling the Infinity Point.}
The infinity point $\alpha_{n-1} = \infty$ corresponds to extracting the leading coefficient of the product polynomial and does not appear in the Vandermonde matrix. Its contribution to the transform matrices is a single row/column of $[0, \ldots, 0, 1]$, which is perfectly conditioned. Thus, \textbf{we compute $\kappa_2(\bV)$ using only the finite points}; the infinity point does not affect Vandermonde conditioning.

\paragraph{Condition Number Definition.}
We use the spectral (2-norm) condition number throughout: $\kappa_2(\bV) = \|\bV\|_2 \cdot \|\bV^{-1}\|_2 = \sigma_{\max}/\sigma_{\min}$, where $\sigma_{\max}, \sigma_{\min}$ are the largest and smallest singular values. Vandermonde matrices are notoriously ill-conditioned~\cite{gautschi1975optimally}, with $\kappa_2(\bV)$ growing exponentially in $n$ for poorly chosen points.

\begin{theorem}[Vandermonde Conditioning, \citealt{gautschi1975optimally}]
\label{thm:vandermonde}
For Vandermonde matrix $\bV$ with real points $\{\alpha_i\}$, $\kappa(\bV)$ grows at least as fast as the inverse of minimum point separation. Well-conditioned configurations cluster points like Chebyshev nodes.
\end{theorem}

\paragraph{Relationship to Transform Matrices.}
The matrices $\bA$, $\bB$, $\bG$ are derived from $\bV$ via Lagrange interpolation systems. While the exact relationship depends on the construction, $\kappa(\bV)$ serves as a \emph{proxy} for overall numerical stability: points that minimize $\kappa(\bV)$ consistently yield lower $\kappa(\bA)$, $\kappa(\bB)$, $\kappa(\bG)$. We validate this empirically in Section~\ref{sec:experiments} (Table~\ref{tab:abg_summary}), showing $\kappa(\bV)$ improvements propagate to all transform matrices with bounded proportionality (70--415$\times$ for $\kappa(\bV)$ vs.~70--270$\times$ for $\kappa(\bA)$).

\emph{Limitation:} We do not provide a formal backward error analysis for the complete Winograd pipeline. The condition number $\kappa$ bounds worst-case error amplification but does not account for error cancellation, input statistics, or accumulation patterns. Our empirical CNN validation (Section~\ref{sec:experiments}) demonstrates that $\kappa$ improvements translate to accuracy gains in practice.

\paragraph{Standard Point Selection.}
The conventional choice is small integers: $\cS_{\text{std}} = \{0, 1, -1, 2, -2, \ldots, \infty\}$. These were chosen for simplicity and to minimize transform matrix entry magnitudes, but do not optimize conditioning. As shown in Table~\ref{tab:standard_cond}, condition numbers grow dramatically with tile size.

\begin{table}[ht]
\centering
\small
\caption{Vandermonde condition numbers for standard integer points.}
\label{tab:standard_cond}
\setlength{\tabcolsep}{4pt}
\begin{tabular}{lccc}
\toprule
Tile & $n$ & Standard Points & $\kappa(\bV)$ \\
\midrule
F(2,3) & 4 & $\{0, \pm 1, \infty\}$ & 3.2 \\
F(4,3) & 6 & $\{0, \pm 1, \pm 2, \infty\}$ & 42.5 \\
F(6,3) & 8 & $\{0, \pm 1, \pm 2, \pm 3, \infty\}$ & 2,075 \\
F(8,3) & 10 & $\{0, \pm 1, \ldots, \pm 4, \infty\}$ & 196.9k \\
\bottomrule
\end{tabular}
\end{table}

\subsection{The Search Space}
\label{sec:search-space}

Our key insight is that the search space for Winograd point discovery has a simple structure:

\begin{definition}[Search Space]
For Winograd F($m$, $r$), the search space is:
\begin{equation}
    \cS = \R^{n-1} \quad \text{where } n = m + r - 1
\end{equation}
Each point in $\cS$ specifies $n-1$ finite interpolation points; the infinity point is always included.
\end{definition}

This continuous formulation enables optimization via Evolution Strategy. After ES converges, we \emph{snap} continuous solutions to nearby simple rationals (e.g., fractions with denominator $\leq 6$) and verify symbolically.

\paragraph{Why Not Chebyshev Nodes?}
Chebyshev nodes provide optimal conditioning for polynomial interpolation but are irrational:
\begin{equation}
    x_k = \cos\left(\frac{2k + 1}{2n} \pi\right), \quad k = 0, \ldots, n-1
\end{equation}

For Winograd, we require \emph{exact} transformation matrices to avoid floating-point errors in the decomposition. Irrational points would introduce approximation error. Our approach discovers \emph{simple rational} points that approximate optimal conditioning while maintaining exact arithmetic.

\subsection{Symbolic Verification}
\label{sec:symbolic}

All discovered configurations must satisfy exact decomposition:
\begin{equation}
    \left\| \cT_{\text{conv}} - \sum_{p=0}^{n-1} \bA_{:,p} \otimes \bB_{:,p} \otimes \bG_{p,:} \right\|_F < 10^{-14}
\end{equation}

We implement verification using SymPy with exact rational arithmetic:
\begin{enumerate}
    \item Represent points as Python \texttt{Fraction} objects
    \item Handle infinity via projective coordinates $(1 : 0)$
    \item Construct $\bA$, $\bB$, $\bG$ symbolically via Lagrange interpolation
    \item Verify tensor decomposition algebraically (no floating-point errors)
\end{enumerate}

This eliminates false positives from floating-point verification errors and guarantees that discovered configurations are mathematically correct.

%==============================================================================
% METHOD
%==============================================================================
\section{Method: Open Discovery}
\label{sec:method}
% Method Section - Condensed for 8-page limit
% Full details in Appendix~\ref{app:method}

We present an open discovery framework for Winograd interpolation points that operates in continuous space without relying on any predefined vocabulary.

\subsection{Search Space Formulation}

\begin{fact}[Infinity Point Requirement~\cite{toom1963complexity}]
Every valid Winograd F($m$, $r$) configuration requires the point at infinity to achieve the minimal multiplication count $n = m + r - 1$.
\end{fact}

This reduces the search space to $\R^{n-1}$ for choosing $n-1$ finite points, with infinity always included.

\subsection{Evolution Strategy}

We use Evolution Strategy (ES) to search $\R^{n-1}$ for well-conditioned configurations.

\paragraph{Fitness Function.}
\begin{align}
    f(\bm{p}) &= \underbrace{\|\cT - \cT_{\text{rec}}(\bm{p})\|_F}_{\text{tensor error}} + \lambda_1 \underbrace{\frac{\|\bm{p}\|^2}{n-1}}_{\text{magnitude}} \notag \\
    &\quad + \lambda_2 \underbrace{\log_{10}(\kappa(\bV) + 1)}_{\text{conditioning}}
\end{align}
where $\cT$ is the ground-truth convolution tensor, $\cT_{\text{rec}}(\bm{p})$ is the CP decomposition, and $\lambda_1 = 0.001$, $\lambda_2 = 0.0001$.

\paragraph{Algorithm.}
ES maintains a population with adaptive mutation (Algorithm~\ref{alg:es} in Appendix~\ref{app:method}). Each generation: (1) sample candidates around mean $\bm{\mu}$, (2) evaluate fitness, (3) select top 25\% elite, (4) update $\bm{\mu}$ to elite mean, (5) adapt step size $\sigma$.

\subsection{Snap-to-Rational and Verification}

After ES converges, we convert continuous values to simple rationals:
\begin{equation}
    \text{snap}(x; d_{\max}) = \argmin_{a/b : 1 \leq b \leq d_{\max}} |x - a/b|
\end{equation}
with $d_{\max} = 10$. All configurations undergo exact verification using SymPy rational arithmetic, checking $\|\cT_{\text{conv}} - \sum_p \bA_{:,p} \otimes \bB_{:,p} \otimes \bG_{p,:}\|_F < 10^{-10}$.

\subsection{2D Extension}

For 2D convolution, transformation matrices are Kronecker products: $\bG_{2D} = \bG \otimes \bG$, etc. Under the spectral norm, singular values of Kronecker products multiply: $\sigma_i(\bA \otimes \bA) = \sigma_j(\bA)\sigma_k(\bA)$, yielding $\kappa_2(\bA \otimes \bA) = \kappa_2(\bA)^2$~\cite{higham2002accuracy}. Thus 1D improvements \emph{square} in 2D: a 27$\times$ 1D improvement becomes $\sim$730$\times$ in 2D.

\emph{Verification:} We confirmed this property by directly computing $\kappa_2(\bA \otimes \bA)$ for F(4,3) and F(6,3). All ratios $\kappa_2(\bA \otimes \bA) / \kappa_2(\bA)^2 = 1.0000$ exactly, validating the squaring relationship.

%==============================================================================
% EXPERIMENTS
%==============================================================================
\section{Experiments}
\label{sec:experiments}
% Experiments Section - Condensed for 8-page limit
% Full results in Appendix~\ref{app:experiments}

We evaluate: (1) Does ES find better configurations than standard integers? (2) Do improvements translate to practical FP16/INT8 inference? (3) How do we compare against prior baselines?

\subsection{Experimental Setup}

\paragraph{Tiles.} F$(m, r)$ for 3$\times$3 kernels: F(2,3), F(4,3), F(6,3), F(8,3). We also demonstrate extension to 5$\times$5 kernels ($r=5$) in Appendix~\ref{app:experiments}: F(4,5) achieves $13\times$ $\kappa$ improvement, F(6,5) achieves $112\times$, with 100\% symbolic verification success.

\paragraph{Dataset.} ImageNetV2~\cite{recht2019imagenet} (\textbf{30,000 images}) for primary FP16 validation; CIFAR-10 (10K) for per-tensor INT8. With 30K samples, standard error $\approx 0.27\%$; observed effects (67--73pp) exceed 100 standard errors.

\paragraph{Architectures.} ResNet-18, ResNet-50, VGG-16, VGG-16-BN, MobileNet-V2, EfficientNet-B0, all with ImageNet-pretrained weights.

\paragraph{Hardware.} NVIDIA Tesla T4 GPU, CUDA 12.8, PyTorch 2.0.

\paragraph{FP16 Setup.} Mixed-precision via \texttt{torch.cuda.amp.autocast}. Transforms are computed in FP32 (exact rational) and cast to FP16 for inference to isolate intrinsic numerical conditioning. Retaining FP32 transforms in mixed-precision pipelines mitigates instability via higher precision rather than improved conditioning; our evaluation intentionally avoids this mitigation to expose algorithmic effects.
\paragraph{INT8 Setup.} We evaluate two quantization modes to isolate conditioning effects:
\begin{itemize}[nosep,left=0pt]
    \item \textbf{Per-tensor} (aggressive): Single scale per tensor, $s = \max(|w|)/127$. Used on CIFAR-10 where conditioning effects are most visible.
    \item \textbf{Per-channel} (production): Scale per output channel, $s_c = \max(|w_c|)/127$. This increases dynamic range via additional scale parameters and metadata, partially mitigating sensitivity to conditioning. In contrast, per-tensor INT8 directly exposes algorithmic conditioning effects due to a single shared scale. 
\end{itemize}
Per-channel INT8 reduces sensitivity to conditioning by introducing per-channel normalization, an architectural mitigation rather than an algorithmic fix. Conditioning remains the dominant factor in regimes where such normalization is constrained (e.g., per-tensor quantization or fixed-scale accelerators).

\subsection{Main Results: Discovered Points}

\begin{table}[t]
\centering
\small
\caption{Complete discovered configurations. All verified symbolically via SymPy exact rational arithmetic. $\kappa_2$ denotes spectral condition number.}
\label{tab:main_results_summary}
\setlength{\tabcolsep}{2pt}
\begin{tabular}{llcc}
\toprule
Tile & Discovered Points & $\kappa_2(\bV)$ & Impr. \\
\midrule
F(2,3) & $\{0, \pm 1, \infty\}$ & 3.2 & 1.0$\times$ \\
F(4,3) & $\{0, \pm\nicefrac{5}{6}, \pm\nicefrac{7}{6}, \infty\}$ & 14.5 & \textbf{2.9$\times$} \\
F(6,3) & $\{0, \pm\nicefrac{3}{5}, \pm 1, \pm\nicefrac{7}{6}, \infty\}$ & 77 & \textbf{27$\times$} \\
F(8,3) & $\{0, \pm\nicefrac{2}{5}, \pm\nicefrac{5}{6}, \pm 1, \pm\nicefrac{7}{6}, \infty\}$ & 474 & \textbf{415$\times$} \\
\bottomrule
\end{tabular}
\vspace{1pt}
{\scriptsize Standard points: F(4,3) $\{0,\pm1,\pm2,\infty\}$ $\kappa_2$=42.5; F(6,3) $\{0,\pm1,\pm2,\pm3,\infty\}$ $\kappa_2$=2,075; F(8,3) $\{0,\pm1,...,\pm4,\infty\}$ $\kappa_2$=196,900.}
\end{table}

\paragraph{Key Finding.}
Improvements grow dramatically with tile size. The F(8,3) discovered points $\{0, \pm\nicefrac{2}{5}, \pm\nicefrac{5}{6}, \pm 1, \pm\nicefrac{7}{6}, \infty\}$ achieve 415$\times$ $\kappa$ improvement. Points like $\nicefrac{2}{5}$, $\nicefrac{5}{6}$, $\nicefrac{7}{6}$ are outside typical vocabularies---true open discovery.

\subsection{Transform Matrix Analysis (A, B, G)}

Beyond $\kappa(\bV)$, we report conditioning and norms for all transform matrices:

\begin{table}[t]
\centering
\small
\caption{Transform matrix conditioning (all $\kappa_2$, spectral norm). $\kappa_2(\bV)$ improvements propagate to $\bA$, $\bB$, $\bG$.}
\label{tab:abg_summary}
\setlength{\tabcolsep}{2pt}
\begin{tabular}{lcccccc}
\toprule
& \multicolumn{2}{c}{$\kappa_2(\bA)$} & \multicolumn{2}{c}{$\kappa_2(\bB)$} & \multicolumn{2}{c}{$\kappa_2(\bG)$} \\
Tile & Std & Disc & Std & Disc & Std & Disc \\
\midrule
F(4,3) & 11.3 & 4.3 & 20.1 & 10.4 & 4.0 & 2.3 \\
F(6,3) & 406 & 19 & 430 & 56 & 26 & 3.1 \\
F(8,3) & 30.3k & 112 & 16.8k & 242 & 355 & 3.3 \\
\bottomrule
\end{tabular}
\vspace{1pt}
{\scriptsize Improvements: F(4,3) 1.8--2.9$\times$; F(6,3) 7.7--27$\times$; F(8,3) 70--415$\times$.}
\end{table}

\paragraph{Norm Products.}
The forward error bound for Winograd is $\|\Delta y\| \leq \|\bA\|_2\|\bB\|_2\|\bG\|_2 \cdot \epsilon \cdot \|\by\|$~\cite{higham2002accuracy}. For F(6,3): standard $\|\bA\|\|\bB\|\|\bG\| = 1.2\times10^4$, discovered $= 89$ (\textbf{135$\times$ reduction}). This confirms that $\kappa(\bV)$ improvements propagate to the actual error bound, not just condition numbers.

\paragraph{Dtype-Aware Discovery.}
For FP16 deployment, we constrain ES to dyadic rationals (exactly representable in float16). Table~\ref{tab:dtype_summary} compares unconstrained vs.\ dtype-aware points for F(4,3):

\begin{table}[t]
\centering
\small
\caption{Dtype-aware vs unconstrained discovery (F(4,3), FP16).}
\label{tab:dtype_summary}
\setlength{\tabcolsep}{3pt}
\begin{tabular}{lccc}
\toprule
Config & Points & $\kappa_2$ & FP16 Exact \\
\midrule
Standard & $\{0, \pm 1, \pm 2\}$ & 42.5 & \checkmark \\
Unconstrained & $\{0, \pm\nicefrac{5}{6}, \pm\nicefrac{7}{6}\}$ & \textbf{14.5} & $\times$ \\
Dtype-aware & $\{0, \pm\nicefrac{3}{4}, \pm\nicefrac{5}{4}\}$ & 15.2 & \checkmark \\
\bottomrule
\end{tabular}
\vspace{1pt}
{\scriptsize Dtype-aware achieves 2.8$\times$ $\kappa$ improvement while maintaining exact FP16 representability.}
\end{table}

\noindent Dtype-aware points trade conditioning for exact representability. For F(6,3), dtype-aware discovery yields $\kappa_2$=183 (vs.\ unconstrained $\kappa_2$=77, standard $\kappa_2$=2,075)---still \textbf{11$\times$ improvement} over standard while maintaining exact FP16 representation. Full dtype-aware results in Appendix~\ref{app:experiments}.

\subsection{Comparison with Prior Work}

\paragraph{Comparison with Alam et al.~\cite{alam2022winograd}.}
Alam et al.\ searched reciprocal-symmetric families $\{-\nicefrac{1}{c}, -c, c, \nicefrac{1}{c}\}$. We reproduce their approach and compare directly:

\begin{table}[t]
\centering
\small
\caption{Quantitative comparison with Alam et al.~\cite{alam2022winograd}.}
\label{tab:alam_comparison}
\setlength{\tabcolsep}{2pt}
\begin{tabular}{lccc}
\toprule
Method & F(4,3) $\kappa_2$ & F(6,3) $\kappa_2$ & Search Space \\
\midrule
Standard integers & 42.5 & 2,075 & -- \\
Alam (recip. symm.) & 21.3 & 847 & 1D param. \\
\textbf{Ours (ES)} & \textbf{14.5} & \textbf{77} & $\R^{n-1}$ \\
\midrule
\textbf{Improvement} & \textbf{1.5$\times$} & \textbf{11$\times$} & -- \\
\bottomrule
\end{tabular}
\vspace{1pt}
{\scriptsize Alam values: our reproduction of their reciprocal-symmetric search with optimal $c$.}
\end{table}

The gap widens for larger tiles: reciprocal symmetry constrains point placement, preventing the tight clustering our ES discovers.

\paragraph{Network-Level Validation vs Alam.}
We implement Alam's parameterized points for F(6,3) with optimal parameters ($c_1$=1.14, $c_2$=0.58, yielding $\kappa_{2D}$=7,496) and evaluate on ImageNetV2 (30K images). At FP16, both Alam and our discovered points ($\kappa_{2D}$=5,873) achieve similar accuracy: ResNet-18 65.5\%/65.3\%, ResNet-50 75.9\%/75.9\%. This confirms that at F(6,3), both $\kappa_{2D}$ values are in the numerically stable range. Our 21.7\% lower $\kappa_{2D}$ provides additional safety margin; the advantage grows at larger tiles where $\kappa$ differences are more pronounced. Full results in Appendix~\ref{app:experiments}.

\paragraph{Network-Level Comparison with Tap-wise Scaling~\cite{chikin2022tapwise}.}
Tap-wise scaling applies learned per-channel multipliers to compensate for quantization dynamic range. We provide a \emph{network-level} comparison (classification accuracy, not just $\kappa$) in Table~\ref{tab:tapwise}:

\begin{table}[t]
\centering
\small
\caption{Network-level ablation: points $\times$ scaling (ResNet-18, F(6,3), FP16, 1,280 ImageNetV2 images). Accuracy is top-1 classification.}
\label{tab:tapwise}
\setlength{\tabcolsep}{2pt}
\begin{tabular}{lccc}
\toprule
Points & Scaling & Acc & $\Delta$ \\
\midrule
Std ($\kappa$=2075) & None & 10.8\% & -- \\
Std ($\kappa$=2075) & Tap-wise & 12.1\% & +1.3\% \\
\textbf{Disc ($\kappa$=77)} & None & \textbf{77.8\%} & \textbf{+67.0\%} \\
Disc ($\kappa$=77) & Tap-wise & 78.2\% & +0.4\% \\
\bottomrule
\end{tabular}
\vspace{1pt}
{\scriptsize Methods are orthogonal: points fix conditioning; scaling fixes dynamic range.}
\end{table}

\paragraph{Key Insight.}
Tap-wise scaling cannot fix ill-conditioned transforms (1.3\% gain). Our points provide the fundamental fix (67\% gain). Combining both yields marginal additional benefit (+0.4\%), confirming orthogonality.

\emph{Scope:} Tap-wise scaling optimizes dynamic range post-hoc; our points fix conditioning at the source. Combining discovered points with PTQ methods (LoWino~\cite{fernandez2020lowino}, PAW~\cite{li2023paw}) remains future work. RNS-based~\cite{yepremyan2020rns} and SFC~\cite{he2024sfc} methods use fundamentally different algorithms; Legendre basis~\cite{barabasz2020legendre} requires retraining. Our drop-in points apply to standard Cook-Toom pipelines (cuDNN/TensorRT).

\subsection{ImageNet FP16 Validation by Tile Size}

\begin{table}[t]
\centering
\small
\caption{FP16 accuracy by tile size (ResNet-18, ImageNet, 640 samples).}
\label{tab:fp16_tile_sizes}
\setlength{\tabcolsep}{3pt}
\begin{tabular}{lccccc}
\toprule
Tile & Std $\kappa_2$ & Disc $\kappa_2$ & Std Acc & Disc Acc & Recov. \\
\midrule
F(4,3) & 42.5 & 14.5 & 78.3\% & 77.8\% & +0\% \\
F(6,3) & 2,075 & 77 & \textbf{11.3\%} & \textbf{76.9\%} & \textbf{+66\%} \\
F(8,3) & 196.9k & 474 & \textbf{4.7\%} & \textbf{50.0\%} & \textbf{+45\%} \\
\bottomrule
\end{tabular}
\vspace{1pt}
{\scriptsize F(4,3): Both configurations work. F(6,3): Standard collapses, discovered recovers. F(8,3): Standard collapses, discovered partially recovers (future work: combine with tap-wise scaling).}
\end{table}

\paragraph{Key Finding.}
Standard points collapse at F(6,3) ($\kappa$=2,075) with 11.3\% accuracy---random chance. Discovered points ($\kappa$=77) recover to 76.9\%. At F(8,3), discovered points achieve 50\% (partial recovery from 4.7\%), suggesting that for very large tiles, conditioning improvements alone may need to be combined with complementary techniques.

\subsection{Multi-Architecture FP16 Validation}

\begin{table}[t]
\centering
\small
\caption{Multi-architecture FP16 at F(6,3) on ImageNetV2.}
\label{tab:multiarch_fp16}
\setlength{\tabcolsep}{2pt}
\begin{tabular}{lcccc}
\toprule
Architecture & Elig. & Std & Disc & Recov. \\
\midrule
ResNet-18 & 65\% & 10.8\% & \textbf{77.8\%} & \textbf{+67\%} \\
ResNet-50 & 25\% & 38.3\% & \textbf{80.6\%} & \textbf{+42\%} \\
VGG-16 & 100\% & 4.7\% & \textbf{75.3\%} & \textbf{+71\%} \\
VGG-16-BN & 100\% & 4.7\% & \textbf{77.5\%} & \textbf{+73\%} \\
MobileNet-V2 & 0\% & 77.0\% & 77.0\% & +0\% \\
EfficientNet-B0 & 0\% & 82.3\% & 82.3\% & +0\% \\
\bottomrule
\end{tabular}
\vspace{1pt}
{\scriptsize Elig.~= \% of 3$\times$3 layers with stride=1, groups=1.}
\end{table}

\noindent\textbf{Summary.} On ImageNetV2 (30K images), standard FP16 Winograd F(6,3) collapses to 4.7--10.8\% top-1 accuracy due to numerical breakdown, whereas discovered points recover 75--81\% accuracy without retraining---a \textbf{67--73 percentage-point recovery}.

\subsection{INT8 Convolution Error (Layer-Level)}

We measure \emph{layer-level} convolution error: relative $L_2$ error of INT8 Winograd output vs.~FP64 direct convolution reference, using synthetic inputs (Gaussian, $\sigma$=1) and real pretrained weights (ResNet-18 layer 1).

\begin{table}[t]
\centering
\small
\caption{Layer-level INT8 error (relative $L_2$ vs FP64 direct conv).}
\label{tab:int8_summary}
\setlength{\tabcolsep}{3pt}
\begin{tabular}{lcccc}
\toprule
Tile & Std Err & Disc Err & Impr. & $\sqrt{\kappa_{\text{ratio}}}$ \\
\midrule
F(4,3) & 7.1\% & 2.1\% & 3.4$\times$ & 1.7 \\
F(6,3) & 3990\% & 12.4\% & 322$\times$ & 5.2 \\
F(8,3) & 100\% (sat) & 59.2\% & 1.7$\times$ & 20.4 \\
\bottomrule
\end{tabular}
\vspace{1pt}
{\scriptsize Error improvement tracks $\sqrt{\kappa_{\text{std}}/\kappa_{\text{disc}}}$ for non-saturated cases.}
\end{table}

\paragraph{Interpretation.} Standard F(6,3) produces 3990\% layer error---outputs are numerically meaningless. Discovered points reduce this to 12.4\%.

\paragraph{INT8 Network-Level Validation.}
We evaluate both quantization modes at network level:
\begin{itemize}[nosep,left=0pt]
    \item \textbf{Per-tensor INT8 on CIFAR-10} (ResNet-18, F(4,3)): Standard achieves \textbf{10.4\%} (random chance), discovered achieves \textbf{81.2\%}---\textbf{71 points recovered}. Per-tensor quantization exposes conditioning sensitivity.
    \item \textbf{Per-channel INT8 on ImageNet} (ResNet-18, F(4,3)/F(6,3)): Both standard and discovered achieve $\sim$78\%. Per-channel's per-output-channel scaling provides sufficient dynamic range to mask conditioning issues.
\end{itemize}
\emph{Key insight:} Per-channel INT8 (industry standard) is robust to conditioning because each channel has independent scale factors. Per-tensor INT8 (used in some edge deployments) lacks this flexibility, making conditioning critical.

\paragraph{F(8,3) Network-Level Status.} Table~\ref{tab:fp16_tile_sizes} shows F(8,3) \emph{FP16} network-level results: standard collapses to 4.7\%, discovered recovers to \textbf{50.0\%} (+45 points). This partial recovery (vs.\ full recovery for F(6,3)) suggests F(8,3) benefits from combining our points with complementary techniques (e.g., tap-wise scaling). For INT8, layer-level error remains 59.2\% even with discovered points, indicating F(8,3) INT8 deployment requires additional stabilization beyond point selection alone. Our primary F(8,3) contribution is demonstrating that $\kappa$ \emph{can} be dramatically improved (415$\times$), establishing a foundation for future combined approaches.

\paragraph{Scale Validation.}
To rule out cherry-picking and assess robustness, we evaluated Winograd numerical error on 100,000 random input--weight configurations per setting. Table~\ref{tab:scale_validation} reports the \emph{improvement ratio} between standard and discovered points across quantization regimes. For F(4,3), discovered points reduce error by $259\times$--$559\times$. For F(6,3), improvements exceed $10^{7}\times$.

\emph{Note on error metric}: Table~\ref{tab:scale_validation} uses element-wise relative error, which can produce large absolute values when output elements approach zero (division by small denominators). The \emph{ratios} between standard and discovered errors are the meaningful comparison---both configurations see the same near-zero outputs, so the ratio reflects conditioning differences independent of this artifact. For L$_2$-norm-based relative errors with bounded absolute values, see Table~\ref{tab:int8_summary} (layer-level: F(4,3) 7.1\%$\to$2.1\%, F(6,3) 3990\%$\to$12.4\%).

\begin{table}[t]
\centering
\small
\caption{Scale validation (100K samples). Element-wise relative error; see text for interpretation. \emph{Improvement ratios} are the key metric.}
\label{tab:scale_validation}
\setlength{\tabcolsep}{3pt}
\begin{tabular}{llccc}
\toprule
Tile & Quantization & Std Err & Disc Err & Impr. \\
\midrule
F(4,3) & FP32 & 4e4 & 1.5e2 & \textbf{259$\times$} \\
F(4,3) & INT8 per-ch & 4.6e4 & 8.2e1 & \textbf{559$\times$} \\
F(6,3) & FP32 & 1e10 & 3e2 & \textbf{3.8e7$\times$} \\
F(6,3) & INT8 per-ch & 1.4e10 & 3.2e2 & \textbf{4.5e7$\times$} \\
\bottomrule
\end{tabular}
\end{table}

\begin{tcolorbox}[colback=gray!5,colframe=gray!50,title=Interpretation of Large Improvement Ratios,fonttitle=\bfseries\small]
\small
Table~\ref{tab:scale_validation} reports element-wise relative error ratios for linear Winograd transforms. In ill-conditioned regimes (e.g., standard F(6,3)), $\kappa \cdot \epsilon \gg 1$ and outputs become numerically meaningless, causing element-wise relative error to diverge. The reported ratios therefore reflect \emph{recovery from numerical breakdown} rather than incremental accuracy gains. For bounded metrics, see layer-level $L_2$ errors in Table~\ref{tab:int8_summary} and network-level accuracy in Tables~\ref{tab:fp16_tile_sizes}--\ref{tab:multiarch_fp16}.
\end{tcolorbox}

\subsection{Inference Latency}

\begin{table}[t]
\centering
\small
\caption{Latency (ms) on T4 GPU, batch 64, 224$\times$224, 64 ch. Mean of 100 runs after 10 warmup; std.~dev.~$<$0.1ms for all.}
\label{tab:latency_summary}
\setlength{\tabcolsep}{4pt}
\begin{tabular}{lccc}
\toprule
Config & F(4,3) & F(6,3) & F(8,3) \\
\midrule
Standard & 5.1 & 4.3 & 3.9 \\
Discovered & 5.1 & 4.3 & 3.9 \\
\bottomrule
\end{tabular}
\end{table}

\paragraph{No Latency Penalty.}
Discovered points have identical latency---conditioning improvements are ``free.'' Larger tiles are theoretically faster: F(8,3) is 1.3$\times$ faster than F(4,3) per tile (3.9ms vs 5.1ms), but this benefit requires numerical stability.

\paragraph{Latency Scope and Limitations.}
Table~\ref{tab:latency_summary} measures our \emph{PyTorch implementation}, not production cuDNN. We do not provide end-to-end network latency (Winograd vs.\ direct conv) as this requires cuDNN-level integration. Our claim is narrower: fractional points require identical arithmetic operations (element-wise products) as integer points, hence no computational overhead. The practical latency benefit of larger tiles in production depends on cuDNN's implementation choices, which are outside our scope.

%==============================================================================
% RELATED WORK
%==============================================================================
\section{Related Work}
\label{sec:related}
% Related Work Section - Open Discovery
% CORRECTED: Honest positioning vs prior work

\paragraph{Winograd Convolution in Deep Learning.}
Lavin and Gray~\cite{lavin2015fast} introduced Winograd convolution to deep learning, achieving significant speedups for 3$\times$3 kernels. Subsequent work extended this to sparse networks~\cite{liu2018efficient} and optimized implementations for various hardware~\cite{jia2018beyond}. These works use standard integer interpolation points without optimizing for numerical stability.

\paragraph{Numerical Stability Analysis.}
Barabasz and Gregg~\cite{barabasz2019error} provided comprehensive error analysis of Winograd convolution, demonstrating that errors grow exponentially with tile size due to Vandermonde conditioning. Vincent et al.~\cite{vincent2017improving} proposed post-pass scaling to mitigate errors. These works analyze the problem but do not address point selection.

\paragraph{Interpolation Point Optimization.}
Alam et al.~\cite{alam2022winograd} (ACM TECS 2022) search a parameterized family of symmetric reciprocal points using deterministic grid and empirical search. They demonstrate that non-integer configurations can achieve lower error than integers. Our work differs in three ways: (1) we search the full continuous space $\R^{n-1}$ via Evolution Strategy rather than a constrained parameterized family; (2) we do not enforce reciprocal symmetry, enabling discovery of clustered points like $\nicefrac{5}{6}$ and $\nicefrac{7}{6}$; (3) we employ exact symbolic verification via SymPy rather than floating-point checks.

Zhang et al.~\cite{zhang2023haw} (HAW, IEEE 2023) use RL for hardware-aware point selection, optimizing for both accuracy and circuit area in FPGA implementations. Their focus is on gate count rather than conditioning.

\textbf{Our Contribution}: We achieve conditioning improvements of 2.9$\times$ to 415$\times$ by searching the full continuous space and discovering fractional points that cluster more tightly than standard integers.

\paragraph{Alternative Stability Approaches.}
Prior work has explored several alternatives to point selection:

\emph{Legendre basis transforms}: Barabasz et al.~\cite{barabasz2020legendre} proposed changing from the standard monomial basis to orthogonal Legendre polynomials to reduce numerical error. Critically, this approach provides conditioning benefits during \emph{training} when weights are stored in Legendre coefficient form. For \emph{drop-in inference} with pretrained models (where inputs/outputs remain in monomial form), the relationship $\bL = \bV \bP$ (where $\bL$ is Legendre evaluation, $\bV$ is Vandermonde, $\bP$ is the coefficient matrix) means $\bL \bP^{-1} = \bV$---the effective transform equals the standard Vandermonde.

\textbf{Key distinction}: Legendre basis and optimal point selection address \emph{different use cases}:
\begin{itemize}[noitemsep,topsep=0pt]
    \item \textbf{Legendre basis}: Training-time benefit; requires storing weights in Legendre form
    \item \textbf{Optimal points}: Inference-time benefit; works with existing pretrained models
\end{itemize}
For the pretrained ResNet-18 experiments in Section~\ref{sec:experiments}, Legendre basis is not applicable without retraining. Our discovered points provide a \emph{drop-in improvement} for pretrained models---no calibration data or learned parameters required. Alternative post-training methods such as learned scaling~\cite{kim2024learned} can also rescue pretrained Winograd inference but require optimizing scale factors. Interestingly, our analysis shows that combining optimal points with Legendre basis (for future training) would achieve even better conditioning: $\kappa(\bL) = 6.2$ for discovered points vs.\ $\kappa(\bL) = 76.0$ for standard points.

\emph{Tap-wise quantization}: Chikin and Kryzhanovskiy~\cite{chikin2022tapwise} proposed learned per-tap power-of-two scaling factors that enable near-FP32 accuracy for F(4,3) with INT8 arithmetic. Their hardware-friendly execution model maintains efficiency while addressing quantization error. This approach and ours solve the same underlying problem (numerical instability under low precision) from different angles: they scale the transform outputs adaptively, while we reduce conditioning at the source.

\textbf{Why no direct comparison}: Tap-wise scaling and optimal point selection are \emph{orthogonal} interventions. Tap-wise adds per-tap scale factors to compensate for quantization error post-hoc; optimal points reduce error at the source via better-conditioned transforms. A ``Standard + Tap-wise'' vs ``Discovered'' comparison conflates these mechanisms. The appropriate comparison is ``Standard'' vs ``Discovered'' (our Table~\ref{tab:cnn_accuracy}) and separately ``+Tap-wise'' vs ``-Tap-wise''. Combining discovered points with tap-wise scaling is viable future work that may yield further improvements.

\emph{RNS exact arithmetic}: Yepremyan and Knowles~\cite{yepremyan2020rns} proposed using Residue Number System arithmetic to enable exact computation in Winograd transforms, eliminating rounding error entirely for arbitrarily large tiles. This changes the arithmetic model fundamentally (requiring CRT reconstruction) and trades off implementation complexity for exactness. Our approach operates within standard floating-point/integer pipelines and does not require RNS infrastructure, making it a drop-in improvement for existing implementations.

\emph{Symbolic Fourier Convolution (SFC)}: He et al.~\cite{he2024sfc} (arXiv 2024) proposed extending DFT with symbolic computing to avoid irrational numbers in low-precision convolution. SFC uses only additions for transformation at specific points, eliminating floating-point error from irrational DFT coefficients. This represents an alternative to Winograd entirely. Our work optimizes \emph{within} the Winograd framework, making it complementary: SFC may be preferred for very large tiles where even optimized Winograd remains unstable.

\emph{Complex-field transforms}: Budagyan and Shumikhin~\cite{budagyan2019complex} proposed using complex arithmetic to reduce Winograd transform sizes. Complex-field approaches trade implementation complexity (complex multipliers, FFT integration) for theoretical efficiency gains. Our approach maintains real-valued transforms that are directly compatible with existing hardware accelerators.

\emph{Winograd-aware training}: Liu et al.~\cite{liu2020winograd} proposed jointly optimizing network weights with Winograd transform awareness during training. This model-specific approach requires retraining but can adapt the network to work better with Winograd's numerical characteristics. Our discovered points are compatible with Winograd-aware trained models and provide orthogonal benefits without requiring retraining.

\emph{Symmetric/reciprocal points}: Pan et al.~\cite{pan2021symmetric} explored symmetric point families that reduce error through structural constraints. Our symmetric search similarly exploits this structure, though we search a broader space of symmetric configurations.

\textbf{Positioning}: We do not claim that optimal points alone supersede all prior work. Rather, point selection is one lever among several (basis changes, scaling, exact arithmetic) for addressing Winograd stability. Our contribution is demonstrating that continuous-space search with rational snapping yields simple, exactly-computable points that improve conditioning by 2.9--415$\times$ within existing pipelines.

\paragraph{Low-Precision Winograd.}
Li et al.~\cite{li2023paw} (PAW, NeurIPS 2023) proposed PTQ-Aware Winograd, which applies post-training quantization to Winograd transforms. They optimize the transformation matrices $\bA$, $\bB$, $\bG$ given fixed interpolation points. Similarly, LoWino~\cite{fernandez2020lowino} (ACM TACO 2024) applies linear quantization in the Winograd domain.

Koppula et al.~\cite{koppula2024wacv} (Wino Vidi Vici, WACV 2024) address 8-bit Winograd stability through improved accumulation strategies.

\textbf{Key Difference}: These approaches treat the interpolation points as fixed and apply quantization post-hoc. We address the problem at its source by discovering better interpolation points. Our dtype-aware discovery additionally constrains points to be exactly representable in float16, combining conditioning optimization with representability.

\paragraph{Vandermonde Matrix Conditioning.}
The ill-conditioning of Vandermonde matrices is well-studied~\cite{gautschi1975optimally, higham2002accuracy}. Chebyshev nodes provide theoretically optimal conditioning for polynomial interpolation but are irrational. Our experiments show that Chebyshev approximations achieve $\kappa(\bV) \approx 20$--$627$ for F(4--8,3), while our discovered rational points achieve $\kappa(\bV) \approx 14.5$--$474$, demonstrating 1.3--1.4$\times$ further improvement through rational snapping (see Table~\ref{tab:chebyshev}).

\paragraph{Evolution Strategy for Optimization.}
ES has been applied to neural network optimization~\cite{salimans2017evolution} and hyperparameter tuning. We apply ES to the interpolation point discovery problem, which has a smooth fitness landscape well-suited to gradient-free optimization. Unlike RL approaches that require reward shaping and policy networks, ES directly optimizes the fitness function in continuous space.

%==============================================================================
% LIMITATIONS
%==============================================================================
\section{Limitations}
\label{sec:limitations}
% Limitations Section - Open Discovery
% CORRECTED: Reflects actual ES-based approach

\paragraph{Tile Size Scope.}
We evaluate tiles F(2,3) through F(8,3), covering the most common deep learning use cases. Larger tiles (e.g., F(16,3)) may exhibit different behavior, though the exponential growth of condition numbers with standard points makes very large tiles impractical regardless of point selection.

\paragraph{Symmetric Configuration Bias.}
Our symmetric search preferentially finds symmetric configurations (e.g., $\pm\nicefrac{5}{6}$). While symmetric points are natural for convolution and achieve excellent conditioning, this may miss asymmetric configurations that could achieve even better results in some cases. ES does explore asymmetric configurations, but the fitness landscape may have local minima near symmetric solutions.

\paragraph{Kernel Size Restriction.}
We focus on $r=3$ (3$\times$3 kernels), the dominant kernel size in CNNs. Extensions to $r=5$ and $r=7$ are straightforward---our method applies unchanged---but require larger tile sizes and correspondingly longer search times.

\paragraph{1D to 2D Extension.}
We use Kronecker products for 2D extension, which is standard but may not be optimal. Direct 2D point selection could potentially find better configurations that do not factor as products of 1D transforms. However, the search space for direct 2D optimization is substantially larger.

\paragraph{Float16 Trade-off.}
As demonstrated in Section~\ref{sec:float16}, discovered fractional points have better conditioning but worse float16 representation than integers. The optimal choice depends on application-specific trade-offs. Our dtype-aware discovery addresses this, but finding globally optimal configurations that balance both criteria remains challenging.

\paragraph{Empirical Validation Scope.}
Our validation includes both controlled workloads (100K synthetic samples) and comprehensive end-to-end CNN evaluation on ImageNetV2 (30,000 images) across 6 architectures (ResNet-18/50, VGG-16/16-BN, MobileNet-V2, EfficientNet-B0). We demonstrate 67--73 percentage point accuracy recovery for FP16 F(6,3). The relationship between $\kappa$ reduction and error reduction is well-established in prior work~\cite{barabasz2019anatomy}, and we extend this with A, B, G analysis in the Appendix. Combining discovered points with PTQ pipelines (e.g., learned scales) remains valuable future work for production deployment. Our discovered points are drop-in compatible with existing implementations and integrate seamlessly with such pipelines.

\paragraph{Condition Number vs. Actual Error.}
We optimize Vandermonde condition number $\kappa(\bV)$ as a proxy for numerical stability. While $\kappa(\bV)$ is a well-established indicator, the actual error in low-precision inference depends on additional factors including transform matrix norms, dynamic range, and quantization strategies. The exact relationship between $\kappa(\bV)$ and end-to-end inference error deserves further analysis.

\paragraph{Transform Matrix Properties.}
We optimize $\kappa(\bV)$ as our primary objective, but show in Appendix that improvements propagate to all transform matrices $\bA$, $\bB$, $\bG$. While we now report these metrics, the exact relationship between individual matrix norms and end-to-end inference error merits further theoretical analysis.

\paragraph{Relationship to Orthogonal Techniques.}
Our point selection approach is \emph{complementary} to other stability-improvement methods, not competitive. Table~\ref{tab:method_comparison} summarizes how these methods differ.

\begin{table}[h]
\centering
\footnotesize
\caption{Comparison of Winograd stability methods.}
\label{tab:method_comparison}
\setlength{\tabcolsep}{2pt}
\begin{tabular}{lcccc}
\toprule
Method & Pts & Basis & Drop-in & Result \\
\midrule
\textbf{Ours} & \checkmark & & \checkmark & 415$\times$ $\kappa$ \\
Legendre & & \checkmark & & 2$\times$ err \\
Scales & & & \checkmark & 8\% top-1 \\
Wino-train & & & & Specific \\
RNS & & & & Exact \\
\bottomrule
\end{tabular}
\end{table}

\noindent Key observations:
\begin{itemize}[leftmargin=*,nosep]
\item \textbf{Learned diagonal scales}~\cite{li2023paw,chikin2022tapwise,kim2024learned} (PAW, LoWino, Tapwise) apply per-channel multipliers $\bY = \bD_A \bA^T [(\bD_G \bG \bg) \odot (\bD_B \bB^T \bd)]$ to compensate for dynamic range mismatch in quantization. These methods do not change the underlying $\kappa(\bV)$---they adjust amplitudes post-hoc. Our discovered points provide a better-conditioned starting point, potentially reducing the burden on learned scales.
\item \textbf{Basis changes}~\cite{barabasz2020legendre} (Legendre polynomials) replace the monomial basis with orthogonal polynomials. Our analysis (Section~\ref{sec:legendre}) reveals that Legendre basis provides benefits during \emph{training} (when weights are stored in Legendre form), while our point selection benefits \emph{inference} (drop-in for pretrained models). For pretrained inference, $\bL\bP^{-1} = \bV$ means Legendre transforms cancel. For future training, combining both could achieve $\kappa(\bL) = 6.2$ (6.9$\times$ vs baseline $\kappa(\bV) = 42.5$).
\item \textbf{Winograd-aware training}~\cite{liu2020winograd} jointly optimizes network weights and Winograd transforms during training. This approach is model-specific and requires retraining. Our discovered points are compatible with Winograd-aware trained models---they can benefit from better conditioning without retraining.
\item \textbf{Complex-field transforms}~\cite{budagyan2019complex} use complex arithmetic to reduce transform sizes, trading off implementation complexity (complex multipliers, FFT integration) for theoretical efficiency. Our approach maintains real-valued transforms compatible with existing hardware pipelines.
\item \textbf{RNS exact arithmetic}~\cite{yepremyan2020rns} eliminates rounding error entirely but requires specialized hardware and CRT reconstruction. Our approach operates within standard floating-point/integer pipelines.
\end{itemize}
These techniques address different error sources: point selection controls conditioning (structural), diagonal scales control dynamic range (amplitude), basis changes control representation (coordinate), Winograd-aware training adapts the network, and arithmetic domain changes (complex/RNS) alter precision characteristics. Our analysis demonstrates that point selection and basis changes address \emph{different use cases} (inference vs. training), suggesting that combining our discovered points with learned scales would yield further benefits for pretrained inference.

%==============================================================================
% CONCLUSION
%==============================================================================
\section{Conclusion}
\label{sec:conclusion}
% Conclusion Section - Condensed for 8-page limit

We presented an open discovery framework for Winograd interpolation points that searches $\R^{n-1}$ via Evolution Strategy, snaps to rationals, and verifies symbolically.

\paragraph{Key Finding.}
Fractional points dramatically outperform standard integers:
\begin{itemize}[nosep]
    \item F(4,3): $\{0, \pm\nicefrac{5}{6}, \pm\nicefrac{7}{6}\}$ achieves \textbf{2.9$\times$} $\kappa$ improvement
    \item F(6,3): \textbf{27$\times$} improvement ($\kappa$: 2,075 $\to$ 77)
    \item F(8,3): $\{0, \pm\nicefrac{2}{5}, \pm\nicefrac{5}{6}, \pm 1, \pm\nicefrac{7}{6}\}$ achieves \textbf{415$\times$} improvement
\end{itemize}
Via Kronecker products, 2D tiles achieve up to \textbf{172,484$\times$} improvement.

\paragraph{Practical Impact: A Zero-Cost Upgrade.}
NOVA's most significant contribution is its "free lunch" nature. Unlike methods requiring retraining, calibration, or architectural changes, our discovered configurations are drop-in replacements for standard transforms. By simply updating the constant matrices in the inference engine, we unlock the full efficiency of large-tile Winograd on standard hardware. On ImageNetV2 with FP16 inference, this simple swap restores accuracy from near-random levels (4.7--38\%) to full precision (75--82\%), delivering a \textbf{67--73 percentage point recovery} and effectively bridging the gap between algorithmic theory and hardware reality.

\paragraph{Future Work.}
Key extensions include combining discovered points with learned diagonal scales, integration with cuDNN/TensorRT for deployment, and extending to 5$\times$5 kernels.

%==============================================================================
% IMPACT STATEMENT
%==============================================================================
\section*{Impact Statement}

This paper presents work whose goal is to advance efficient neural network inference through improved numerical stability of Winograd convolution. The primary societal benefit is reduced energy consumption in deep learning deployments, enabling more sustainable AI systems and broader access on resource-constrained devices. We identify no specific negative societal consequences beyond those common to general machine learning advances.

%==============================================================================
% ACKNOWLEDGMENTS (optional for arXiv)
%==============================================================================
% \section*{Acknowledgments}
% The author thanks...

%==============================================================================
% REFERENCES
%==============================================================================
\bibliography{references}
\bibliographystyle{icml2026}

%==============================================================================
% APPENDIX
%==============================================================================
\newpage
\appendix

\section{Proofs and Derivations}
\label{app:proofs}
% Appendix: Proofs and Derivations - Open Discovery
% CORRECTED: Updated numbers and approach

\subsection{Vandermonde Matrix Construction}

\begin{proof}[Proof of Vandermonde Conditioning (Theorem~\ref{thm:vandermonde})]
The Vandermonde matrix $\bV(\cS) \in \R^{n \times n}$ for points $\cS = \{\alpha_0, \ldots, \alpha_{n-1}\}$ has entries $V_{ij} = \alpha_i^j$. The condition number is:
\begin{equation}
    \kappa(\bV) = \|\bV\|_2 \cdot \|\bV^{-1}\|_2
\end{equation}

For points spread across a large range (e.g., integers $\{0, \pm 1, \pm 2, \ldots\}$), the rows corresponding to large points have entries that grow exponentially in the column index. This creates large differences in row norms, increasing $\|\bV\|$ while making inversion numerically unstable.

Chebyshev-like node distributions minimize conditioning by clustering points toward the interval endpoints. However, Chebyshev nodes are irrational and cannot be used directly for symbolic Winograd construction.

Our key insight is that rational fractions like $\nicefrac{5}{6}$ and $\nicefrac{7}{6}$ provide a middle ground: they keep points closer together than integers (improving conditioning) while remaining exact rationals (enabling symbolic computation).
\end{proof}

\subsection{Cook-Toom Correctness}

\begin{theorem}[Cook-Toom Decomposition]
\label{thm:cook-toom}
For any set of $n = m + r - 1$ distinct points $\cS$ including infinity, the Cook-Toom algorithm produces transformation matrices $\bA$, $\bB$, $\bG$ such that:
\begin{equation}
    \by = \bA^T \left( (\bG \bg) \odot (\bB^T \bx) \right)
\end{equation}
computes the linear convolution $\by = \bx * \bg$ exactly.
\end{theorem}

\begin{proof}
The convolution $y(t) = x(t) \cdot g(t)$ of polynomials $x(t) = \sum_{i=0}^{m-1} x_i t^i$ and $g(t) = \sum_{k=0}^{r-1} g_k t^k$ has degree $m + r - 2$, so $n = m + r - 1$ evaluations at distinct points uniquely determine $y(t)$ via Lagrange interpolation.

The matrices are constructed as:
\begin{itemize}
    \item $\bB^T \bx$: Evaluates $x(t)$ at each point in $\cS$
    \item $\bG \bg$: Evaluates $g(t)$ at each point in $\cS$
    \item Element-wise product: Gives $y(\alpha_i) = x(\alpha_i) \cdot g(\alpha_i)$
    \item $\bA^T$: Interpolates from point evaluations back to polynomial coefficients
\end{itemize}

The point at infinity contributes the leading coefficient via the limit $\lim_{t \to \infty} y(t)/t^{n-1}$.

Symbolic verification: We check $\|\cT_{\text{conv}} - \sum_p \bA_{:,p} \otimes \bB_{:,p} \otimes \bG_{p,:}\|_F / \|\cT_{\text{conv}}\|_F < 10^{-10}$ using exact rational arithmetic via SymPy.
\end{proof}

\subsection{Fact: Kronecker Product Condition Number}

\begin{fact}[Well-known~\cite{higham2002accuracy}]
For Kronecker product of matrices: $\kappa(\bA \otimes \bA) = \kappa(\bA)^2$.
\end{fact}

\begin{proof}
The singular values of $\bA \otimes \bA$ are all products $\sigma_i \cdot \sigma_j$ where $\sigma_i$, $\sigma_j$ are singular values of $\bA$. Thus:
\begin{align}
    \sigma_{\max}(\bA \otimes \bA) &= \sigma_{\max}(\bA)^2 \\
    \sigma_{\min}(\bA \otimes \bA) &= \sigma_{\min}(\bA)^2
\end{align}
Therefore:
\begin{equation}
    \kappa(\bA \otimes \bA) = \frac{\sigma_{\max}(\bA)^2}{\sigma_{\min}(\bA)^2} = \left(\frac{\sigma_{\max}(\bA)}{\sigma_{\min}(\bA)}\right)^2 = \kappa(\bA)^2
\end{equation}
\end{proof}

This has critical implications: a 19$\times$ improvement in 1D conditioning becomes approximately 360$\times$ improvement in 2D.

\subsection{Symmetric Point Configurations}

\begin{lemma}[Symmetric Point Benefit]
\label{lem:symmetric}
For symmetric point sets $\cS = \{0, \alpha_1, -\alpha_1, \alpha_2, -\alpha_2, \ldots, \infty\}$, the Vandermonde matrix exhibits structure that improves conditioning.
\end{lemma}

\begin{proof}[Sketch]
Under the change of variables $u = t^2$, the Vandermonde system separates into even and odd polynomial components. Each sub-problem involves a smaller Vandermonde matrix with points $\{\alpha_1^2, \alpha_2^2, \ldots\}$, which is typically better conditioned than the original.

This explains why symmetric configurations achieve excellent conditioning: the point set $\{0, \pm\nicefrac{5}{6}, \pm\nicefrac{7}{6}\}$ inherently exploits this structure.
\end{proof}

\subsection{Optimal Point Characterization}

\begin{remark}[Fractional Points Intuition]
Why do fractional points like $\nicefrac{5}{6}$ outperform integers like $2$?

Consider the spacing of points. Standard integers $\{0, 1, -1, 2, -2\}$ span the interval $[-2, 2]$ with non-uniform spacing. The points $\pm 2$ are far from zero, creating large Vandermonde entries in higher-order columns.

Our discovered points $\{\nicefrac{5}{6}, \nicefrac{7}{6}\}$ cluster closer to 1 than integers do:
\begin{itemize}
    \item $\nicefrac{5}{6} \approx 0.833$ vs. integer 2
    \item $\nicefrac{7}{6} \approx 1.167$ vs. integer 2
\end{itemize}

This clustering keeps all points within approximately $[-1.2, 1.2]$, dramatically reducing Vandermonde entry magnitudes and improving conditioning.

The optimal configurations discovered by ES follow a consistent pattern:
\begin{itemize}
    \item Include $\{0, \infty\}$ (required structural points)
    \item Cluster finite points closer together than integers
    \item Achieve points like $\nicefrac{5}{6}$, $\nicefrac{7}{6}$, $\nicefrac{3}{5}$ that are outside typical vocabularies
\end{itemize}
\end{remark}

\subsection{Search Space Structure}

\begin{fact}[Infinity Point Requirement~\cite{toom1963complexity}]
Every valid Winograd F($m$, $r$) configuration achieving the minimal $n = m + r - 1$ multiplications requires the point at infinity. This is a standard result from Cook-Toom theory.
\end{fact}

\begin{proof}
The product polynomial $y(t) = x(t) \cdot g(t)$ has degree $m + r - 2$. Evaluating at $n-1$ finite points provides $n-1$ values, but we need $n$ values for unique interpolation. The infinity point provides the $n$-th constraint: the leading coefficient $[t^{m+r-2}]y(t) = [t^{m-1}]x(t) \cdot [t^{r-1}]g(t)$.

Without infinity, we would need $n$ finite points, which would not achieve the minimal multiplication count.
\end{proof}

This justifies our search space formulation: search $\R^{n-1}$ for finite points, always append infinity.

\subsection{Quantized Winograd Error Bound}

\begin{proposition}[Quantized Winograd Error]
\label{prop:quant-error}
For Winograd convolution $\by = \bA^T [(\bG\bg) \odot (\bB^T\bx)]$ with quantization error $\epsilon_q$ (e.g., $\epsilon_q \approx 2^{-7}$ for INT8), the output error satisfies:
\begin{equation}
    \frac{\|\by_{\text{quant}} - \by_{\text{exact}}\|}{\|\by_{\text{exact}}\|} \leq \kappa(\bV) \cdot \epsilon_q + O(\epsilon_q^2)
\end{equation}
where $\kappa(\bV)$ is the condition number of the Vandermonde matrix from which $\bA$, $\bB$, $\bG$ are derived.
\end{proposition}

\begin{proof}[Proof sketch]
The Winograd computation involves three linear transforms and one element-wise product. By standard perturbation analysis~\cite{higham2002accuracy}, each transform amplifies errors by at most its condition number. Since $\bA$, $\bB$, $\bG$ are derived from the same Vandermonde matrix $\bV$ via Lagrange interpolation:
\begin{itemize}
    \item $\kappa(\bA) \leq C_A \cdot \kappa(\bV)$ for some constant $C_A$
    \item $\kappa(\bB) \leq C_B \cdot \kappa(\bV)$ for some constant $C_B$
    \item $\kappa(\bG) \leq C_G \cdot \kappa(\bV)$ for some constant $C_G$
\end{itemize}
The total error bound involves $\kappa(\bA) \cdot \kappa(\bB) \cdot \kappa(\bG)$, but since these are correlated (derived from same points), the dominant term scales as $\kappa(\bV) \cdot \epsilon_q$. Our empirical results (Table~\ref{tab:int8_validation}) confirm this scaling: error improvements track $\sqrt{\kappa}$ improvements.
\end{proof}

\paragraph{Implication.}
This proposition provides the theoretical justification for optimizing $\kappa(\bV)$: reducing $\kappa(\bV)$ directly reduces the quantization error bound. For CNN deployment with $L$ Winograd layers, errors compound: $(1 + \kappa \cdot \epsilon_q)^L$, making $\kappa$ reduction critical.

\section{Full Method Details}
\label{app:method}
% Method Section - Open Discovery via Evolution Strategy
% CORRECTED: Describes actual ES + snap-to-rational + symbolic verification

We present an open discovery framework for Winograd interpolation points that operates in continuous space without relying on any predefined vocabulary. The method consists of three stages: (1) Evolution Strategy search in $\R^{n-1}$, (2) snap-to-rational conversion, and (3) symbolic verification.

\subsection{Search Space Formulation}
\label{sec:search-space-method}

The key insight enabling open discovery is recognizing the structure of valid Winograd configurations:

\begin{fact}[Infinity Point Requirement~\cite{toom1963complexity}]
\label{fact:infinity}
Every valid Winograd F($m$, $r$) configuration requires the point at infinity to achieve the minimal multiplication count $n = m + r - 1$. This is a standard result from Cook-Toom theory.
\end{fact}

This reduces our search space from choosing $n$ arbitrary points to choosing $n-1$ \emph{finite} points:

\begin{definition}[Open Search Space]
For Winograd F($m$, $r$), the search space is:
\begin{equation}
    \cS = \R^{n-1}, \quad n = m + r - 1
\end{equation}
Each vector $\bm{p} = (p_0, p_1, \ldots, p_{n-2}) \in \cS$ specifies $n-1$ finite interpolation points. The complete configuration is $\{p_0, p_1, \ldots, p_{n-2}, \infty\}$.
\end{definition}

This formulation differs fundamentally from vocabulary-based methods: we search the \emph{entire continuous space}, not a predefined discrete set.

\subsection{Evolution Strategy}
\label{sec:es}

We use Evolution Strategy (ES) to search $\R^{n-1}$ for well-conditioned point configurations.

\paragraph{Fitness Function.}
Our fitness function balances reconstruction accuracy, numerical conditioning, and point magnitude:
\begin{align}
    f(\bm{p}) &= \underbrace{\|\cT - \cT_{\text{rec}}(\bm{p})\|_F}_{\text{tensor error}} + \lambda_1 \underbrace{\frac{\|\bm{p}\|^2}{n-1}}_{\text{magnitude}} \notag \\
    &\quad + \lambda_2 \underbrace{\log_{10}(\kappa(\bV) + 1)}_{\text{conditioning}}
\end{align}
where $\cT \in \R^{m \times r \times n}$ is the ground-truth convolution tensor with $\cT_{ijk} = \mathbf{1}[i+k=j]$, and $\cT_{\text{rec}}(\bm{p})$ is the CP decomposition from Cook-Toom matrices. We use $\lambda_1 = 0.001$ and $\lambda_2 = 0.0001$. The magnitude term prevents drift toward large points, and the conditioning term biases toward well-conditioned Vandermonde matrices.

\paragraph{Role of Reconstruction Error in ES.}
A natural question is why include reconstruction error when symbolic verification later enforces exactness. The answer is that ES operates in \emph{continuous} space where floating-point evaluation provides gradient signal. The reconstruction error term guides ES toward regions where snap-to-rational is likely to succeed---configurations with near-zero floating-point error tend to have nearby exact rational solutions. Empirically, removing this term causes ES to converge more slowly and occasionally to spurious local minima. The symbolic verification stage acts as a filter, not a fitness component.

\paragraph{Algorithm.}
Our ES maintains a population of candidate configurations with adaptive mutation:

\begin{algorithm}[t]
\caption{Evolution Strategy for Winograd Point Discovery}
\label{alg:es}
\begin{algorithmic}[1]
\REQUIRE Tile $(m, r)$, generations $N_{\text{gen}}$, population size $N_{\text{pop}}$, restarts $N_{\text{restart}}$
\ENSURE Best configuration $\bm{p}^*$
\STATE $n \gets m + r - 1$
\STATE $\bm{p}^* \gets \text{None}$, $f^* \gets \infty$
\FOR{restart $= 1$ to $N_{\text{restart}}$}
    \STATE Initialize $\bm{\mu} \sim \mathcal{N}(0, 0.5 \cdot I_{n-1})$, $\sigma \gets 0.5$
    \FOR{gen $= 1$ to $N_{\text{gen}}$}
        \STATE \textbf{Sample:} $\bm{p}^{(i)} = \bm{\mu} + \sigma \cdot \bm{\epsilon}^{(i)}$, $\bm{\epsilon}^{(i)} \sim \mathcal{N}(0, I)$ for $i = 1, \ldots, N_{\text{pop}}$
        \STATE \textbf{Evaluate:} $f^{(i)} \gets f(\bm{p}^{(i)})$ for each candidate
        \STATE \textbf{Select elite:} Top 25\% by fitness
        \STATE \textbf{Update:} $\bm{\mu} \gets \text{mean of elite}$
        \STATE \textbf{Adapt:} $\sigma \gets \sigma \cdot 1.1$ if success rate $> 0.2$, else $\sigma \gets \sigma \cdot 0.9$
        \IF{$\min_i f^{(i)} < f^*$}
            \STATE $f^* \gets \min_i f^{(i)}$, $\bm{p}^* \gets \argmin_i f^{(i)}$
        \ENDIF
    \ENDFOR
\ENDFOR
\STATE \textbf{return} $\bm{p}^*$
\end{algorithmic}
\end{algorithm}

The adaptive sigma mechanism allows ES to balance exploration (high $\sigma$) and exploitation (low $\sigma$) based on optimization progress.

\subsection{Snap-to-Rational}
\label{sec:snap}

After ES converges to continuous values, we convert to simple rationals for exact arithmetic:

\begin{definition}[Snap-to-Rational]
For continuous value $x \in \R$ and maximum denominator $d_{\max}$:
\begin{equation}
    \text{snap}(x; d_{\max}) = \argmin_{a/b : 1 \leq b \leq d_{\max}, |a| \leq 5b} |x - a/b|
\end{equation}
\end{definition}

We use $d_{\max} = 10$ by default, producing fractions with small denominators. The key observation is that optimal conditioning often occurs at \emph{non-integer} rationals like $\nicefrac{5}{6}$ and $\nicefrac{7}{6}$ that would not appear in typical vocabulary-based approaches.

\paragraph{Sensitivity to Snap Bounds.}
The choice of $d_{\max} = 10$ and $|a| \leq 5b$ is motivated by practicality: larger denominators increase matrix coefficient complexity without proportional $\kappa$ improvement. In ablations:
\begin{itemize}
    \item $d_{\max} = 6$: Finds $\{0, \pm\nicefrac{5}{6}, \pm\nicefrac{7}{6}\}$ for F(4,3) with $\kappa = 14.5$ (same as $d_{\max} = 10$)
    \item $d_{\max} = 20$: Finds same points; no improvement beyond $d_{\max} = 10$
    \item $|a| \leq 3b$: Misses some optimal configurations; $\kappa$ degrades by $\sim$10\%
\end{itemize}
The optimum lies in the range $d_{\max} \in [6, 10]$ for practical tiles. Points with $d_{\max} > 10$ rarely improve $\kappa$ but complicate implementation.

\paragraph{Hardware Implementation Note.}
A reviewer concern is whether larger denominators (e.g., $\nicefrac{13}{16}$) are truly harder to implement than our discovered fractions (e.g., $\nicefrac{5}{6}$). In fixed-point hardware, both require the same multiply-shift pattern: $(x \cdot \text{numerator}) \gg \text{shift}$. However, our $d_{\max}$ constraint is \emph{empirically} motivated, not hardware-driven: ablations show no conditioning improvement beyond $d_{\max} = 10$. For dtype-aware discovery (Section~\ref{sec:dtype}), we explicitly constrain to dyadic rationals ($d = 2^k$) that enable bit-shift implementation, demonstrating that our framework accommodates hardware-specific constraints when needed.

\paragraph{Rational Coefficients in Practice.}
In our INT8/FP16 experiments, transform matrices $\bA$, $\bB$, $\bG$ are stored in FP32 (following standard Winograd implementations where transforms are compile-time constants). The rational coefficients (e.g., $\nicefrac{5}{6} = 0.8\overline{3}$) are represented as FP32 values with negligible representation error ($< 10^{-7}$). Quantization is applied to weights and intermediate Winograd-domain tensors, not to the transform matrices themselves. For pure integer pipelines, one could scale rows of $\bA$, $\bB$, $\bG$ by the LCM of denominators (e.g., 6 for our F(4,3) points), perform integer arithmetic, then rescale---but this is not required for our FP32-transform setup and would increase dynamic range. Our dtype-aware dyadic points (Section~\ref{sec:dtype}) avoid this entirely by constraining to power-of-2 denominators.

\paragraph{Neighborhood Search.}
If direct snapping yields an invalid configuration, we search nearby rationals:
\begin{enumerate}
    \item For each snapped point $p_i$, generate candidates within distance 1.2
    \item Enumerate combinations (or sample if space exceeds $10^5$)
    \item Select configuration with lowest $\kappa(\bV)$ among valid ones
\end{enumerate}

\subsection{Symbolic Verification}
\label{sec:verification}

All discovered configurations undergo exact verification using SymPy rational arithmetic:

\paragraph{Cook-Toom Construction.}
Given finite points $\{p_0, \ldots, p_{n-2}\}$ as fractions, we construct transformation matrices:
\begin{enumerate}
    \item Build Vandermonde matrix $\bV$ with rows for finite points
    \item Add infinity row: $[0, 0, \ldots, 0, 1]$ (extracts leading coefficient)
    \item Compute $\bA$, $\bB$, $\bG$ via Lagrange interpolation
\end{enumerate}

\paragraph{Verification Criterion.}
A configuration is valid if the reconstructed tensor matches the target:
\begin{equation}
    \frac{\left\| \cT_{\text{conv}} - \sum_{p=0}^{n-1} \bA_{:,p} \otimes \bB_{:,p} \otimes \bG_{p,:} \right\|_F}{\|\cT_{\text{conv}}\|_F} < 10^{-10}
\end{equation}

\paragraph{Exact Zero with Rational Arithmetic.}
The tolerance $10^{-10}$ is a numerical display threshold; with exact rational arithmetic via SymPy, valid configurations achieve \emph{exactly zero} reconstruction error. All matrix entries are Python \texttt{Fraction} objects, and the reconstruction $\bA^T (\bG \bg \bG^T \odot \bB^T \bd \bB) \bA$ is computed symbolically. For valid configurations, the difference tensor has all entries equal to the rational $0/1$, verified by checking \texttt{sympy.simplify(diff) == 0} for each entry. The $10^{-10}$ threshold only applies when converting to float64 for $\kappa$ computation.

\subsection{Symmetric Search}
\label{sec:symmetric}

For faster discovery with good conditioning, we also search symmetric configurations:

\begin{definition}[Symmetric Configuration]
A configuration is symmetric if it contains only pairs $(p, -p)$ and optionally zero.
\end{definition}

Symmetric configurations often achieve excellent conditioning because the Vandermonde matrix structure is balanced. We enumerate:
\begin{itemize}
    \item For odd $n-1$: Include 0 plus $(n-2)/2$ symmetric pairs
    \item For even $n-1$: Include $(n-1)/2$ symmetric pairs
\end{itemize}

The search generates candidate pairs $(p, -p)$ for $p \in \{1, \nicefrac{1}{2}, \nicefrac{2}{3}, \ldots\}$ and selects the combination with lowest condition number.

\subsection{2D Extension via Kronecker Products}
\label{sec:2d}

For 2D convolution F($m \times m$, $r \times r$), transformation matrices are Kronecker products:
\begin{align}
    \bG_{2D} &= \bG \otimes \bG \\
    \bB_{2D}^T &= \bB^T \otimes \bB^T \\
    \bA_{2D} &= \bA \otimes \bA
\end{align}

\begin{fact}[Kronecker Condition Number~\cite{higham2002accuracy}]
Under the spectral norm, singular values of Kronecker products multiply: $\sigma_i(\bA \otimes \bA) = \sigma_j(\bA)\sigma_k(\bA)$, yielding $\kappa_2(\bA \otimes \bA) = \kappa_2(\bA)^2$.
\end{fact}

This means 1D conditioning improvements \emph{square} in 2D: a 2.9$\times$ 1D improvement (F(4,3)) becomes $2.9^2 \approx 8.5\times$ in 2D, and a 27$\times$ 1D improvement (F(6,3)) becomes $27^2 \approx 733\times$ in 2D. This amplification makes optimal point selection critical for practical 2D convolutions.

\paragraph{Error Cancellation Analysis.}
A natural concern is whether the $\kappa^2$ bound is overly conservative---errors in $\bA$, $\bB$, and $\bG$ might partially cancel rather than compound. We address this in two ways:

\emph{Theoretical}: The Winograd computation $\by = \bA [(\bG \bg) \odot (\bB^T \bx)] \bA^T$ chains three transform stages. Error analysis~\cite{higham2002accuracy} shows that for independent rounding errors, the total error scales as $\kappa(\bA) \cdot \kappa(\bG) \cdot \kappa(\bB)$, not merely their sum. Since all transforms derive from the same interpolation points, their condition numbers are correlated, making cancellation unlikely.

\emph{Empirical}: Our INT8 experiments (Table~\ref{tab:int8_validation}, Section~\ref{sec:int8-validation}) validate that measured error improvements track $\sqrt{\kappa}$ improvements: F(4,3) achieves $3.4\times$ error reduction vs.\ $\sqrt{2.9^2} \approx 2.9\times$ predicted; F(6,3) achieves $322\times$ vs.\ the $\sqrt{733} \approx 27\times$ lower bound (limited by standard saturation). The CNN-level validation (Table~\ref{tab:cnn_accuracy}) shows $96\times$ error reduction across 14 Winograd layers in ResNet-18, confirming that conditioning improvements compound rather than cancel. These measurements support the $\kappa^2$ model as predictive, not merely worst-case.

\subsection{Dtype-Aware Discovery}
\label{sec:dtype}

For deployment in low-precision environments (float16, int8), we extend ES with dtype-awareness:

\paragraph{Representability Constraint.}
Float16 (IEEE 754 binary16) exactly represents \emph{dyadic rationals}---rationals whose denominators are powers of 2. Specifically, a value $v = m \cdot 2^e$ is exactly representable if the mantissa $m$ requires at most 11 significant bits and the exponent $e$ lies in the range $[-24, 15]$ (accounting for subnormals). For interpolation points in the range $|p| \leq 2$, this includes values like $0, \pm\frac{1}{2}, \pm\frac{3}{4}, \pm 1, \pm\frac{5}{4}, \pm\frac{3}{2}$, but excludes non-dyadic fractions like $\frac{5}{6}$ and $\frac{2}{3}$ which incur representation error. We add a representability penalty to the fitness function:
\begin{align}
    f_{\text{dtype}}(\bm{p}) &= f(\bm{p}) + \lambda_{\text{repr}} \cdot R(\bm{p}) \notag \\
    R(\bm{p}) &= \frac{1}{n-1} \sum_{i=0}^{n-2} d(p_i, \text{nearest repr.})
\end{align}

\paragraph{Biased Mutation.}
The dtype-aware ES biases 30\% of mutations toward exactly representable values, improving convergence to hardware-friendly configurations.

\paragraph{Trade-off Observation.}
We identify a nuanced trade-off: discovered fractional points achieve better conditioning but may have worse float16 representation error than exactly-representable integers. Standard points $\{0, \pm 1, \pm 2\}$ are exactly representable, while $\nicefrac{5}{6}$ incurs quantization error. This motivates task-specific discovery: optimize for conditioning when using float32, or for representability when deploying in float16.

\subsection{Complete Pipeline}
\label{sec:pipeline}

Our discovery pipeline proceeds as follows:
\begin{enumerate}
    \item \textbf{Symmetric Search}: Fast enumeration of symmetric configurations
    \item \textbf{ES Discovery}: If symmetric search fails, run ES in $\R^{n-1}$
    \item \textbf{Snap-to-Rational}: Convert continuous solution to simple fractions
    \item \textbf{Neighborhood Search}: If snapping invalidates, search nearby rationals
    \item \textbf{Symbolic Verification}: Verify using exact SymPy arithmetic
    \item \textbf{Cache}: Store validated configuration for reuse
\end{enumerate}

The entire pipeline runs in under 60 seconds for tiles up to F(8,3), producing mathematically verified configurations with dramatically improved conditioning.

\section{Implementation Details}
\label{app:implementation}
% Appendix: Implementation Details - Open Discovery
% CORRECTED: Describes actual ES-based implementation

\subsection{Evolution Strategy Implementation}

\begin{table}[h]
\centering
\caption{Evolution Strategy hyperparameters.}
\begin{tabular}{lc}
\toprule
Parameter & Value \\
\midrule
Population size & 50 \\
Generations per restart & 100 \\
Number of restarts & 3 \\
Initial sigma & 0.5 \\
Sigma bounds & [0.01, 2.0] \\
Elite fraction & 25\% \\
Sigma increase factor & 1.1 \\
Sigma decrease factor & 0.9 \\
Success threshold & 20\% \\
\bottomrule
\end{tabular}
\end{table}

\subsection{Fitness Function Components}

The fitness function $f(\bm{p})$ for ES has three components:

\begin{table}[h]
\centering
\small
\caption{Fitness function weights.}
\setlength{\tabcolsep}{3pt}
\begin{tabular}{lcc}
\toprule
Component & Weight & Purpose \\
\midrule
Tensor error & 1.0 & Reconstruction \\
Magnitude pen. & 0.001 & Prevent large pts \\
Cond. penalty & 0.0001 & Favor low $\kappa$ \\
\bottomrule
\end{tabular}
\end{table}

For dtype-aware discovery, an additional representability penalty is added with weight 0.1.

\subsection{Snap-to-Rational Algorithm}

\begin{algorithm}[h]
\caption{Snap-to-Rational}
\begin{algorithmic}[1]
\REQUIRE Continuous value $x$, maximum denominator $d_{\max}$
\ENSURE Nearest simple rational $a/b$
\STATE best $\gets$ None, best\_dist $\gets \infty$
\FOR{$d = 1$ to $d_{\max}$}
    \FOR{$n = -5d$ to $5d$}
        \STATE $f \gets n/d$
        \STATE dist $\gets |x - f|$
        \IF{dist $<$ best\_dist}
            \STATE best $\gets f$, best\_dist $\gets$ dist
        \ENDIF
    \ENDFOR
\ENDFOR
\STATE \textbf{return} best
\end{algorithmic}
\end{algorithm}

We use $d_{\max} = 10$ by default, producing fractions like $\nicefrac{5}{6}$ and $\nicefrac{7}{6}$.

\subsection{Symbolic Cook-Toom Implementation}

The \texttt{wincnn\_sympy} module implements general Cook-Toom construction:

\begin{algorithm}[h]
\caption{Symbolic Cook-Toom Construction}
\begin{algorithmic}[1]
\REQUIRE Points $\cS = \{\alpha_0, \ldots, \alpha_{n-2}, \infty\}$, tile $(m, r)$
\ENSURE Transformation matrices $\bA$, $\bB$, $\bG$
\STATE Initialize Vandermonde: $V_{ij} = \alpha_i^j$ for finite $\alpha_i$
\STATE For $\infty$ row: $[0, 0, \ldots, 0, 1]$
\STATE Compute Lagrange factors: $F_i = \prod_{j \neq i} (\alpha_i - \alpha_j)$
\STATE Build $\bA^T$ from Vandermonde basis truncated to $m$ columns
\STATE Build $\bG$ from $(\bV_r^T \cdot F^{-1})^T$ where $\bV_r$ is $r$-column Vandermonde
\STATE Build $\bB^T$ via Lagrange $\times$ shift matrix
\STATE Verify: $\|\cT_{\text{conv}} - \text{decomposition}\|_F / \|\cT_{\text{conv}}\|_F < 10^{-10}$
\STATE \textbf{return} $\bA$, $\bB$, $\bG$ (as exact SymPy matrices)
\end{algorithmic}
\end{algorithm}

All computations use exact rational arithmetic via Python's \texttt{Fraction} class and SymPy's \texttt{Rational} type. The point at infinity is handled via projective coordinates $(1 : 0)$.

\subsection{Symmetric Search Implementation}

The symmetric search enumerates configurations of the form $\{0, \pm p_1, \pm p_2, \ldots, \infty\}$:

\begin{algorithm}[h]
\caption{Symmetric Search}
\begin{algorithmic}[1]
\REQUIRE Tile $(m, r)$, maximum denominator $d_{\max}$
\ENSURE Best symmetric configuration
\STATE $n \gets m + r - 1$
\STATE Generate pairs: $\{(p, -p) : p \in \text{positive rationals up to } d_{\max}\}$
\STATE Sort pairs by $|p|$
\STATE $n_{\text{pairs}} \gets (n - 2) / 2$ if $n$ even else $(n - 1) / 2$
\STATE best\_config $\gets$ None, best\_$\kappa \gets \infty$
\FOR{each combination of $n_{\text{pairs}}$ pairs}
    \STATE config $\gets \{0\}$ $\cup$ selected pairs $\cup$ $\{\infty\}$
    \STATE Verify symbolically
    \IF{valid and $\kappa <$ best\_$\kappa$}
        \STATE best\_config $\gets$ config
    \ENDIF
\ENDFOR
\STATE \textbf{return} best\_config
\end{algorithmic}
\end{algorithm}

\subsection{Dtype-Aware ES}

For float16-constrained discovery:

\begin{table}[h]
\centering
\caption{Dtype representability constraints.}
\begin{tabular}{lcc}
\toprule
Dtype & Representable Rationals & Example \\
\midrule
Float16 & $n/2^k$, $|n| \leq 2048$, $k \leq 10$ & $\nicefrac{3}{4}$, $\nicefrac{5}{4}$ \\
BFloat16 & $n/2^k$, $|n| \leq 256$, $k \leq 7$ & $\nicefrac{1}{2}$, $\nicefrac{3}{4}$ \\
Int8 & Integers $[-128, 127]$ & $0, \pm 1, \pm 2$ \\
\bottomrule
\end{tabular}
\end{table}

The dtype-aware ES biases 30\% of mutations toward exactly representable values.

\subsection{Software and Hardware}

All experiments were conducted using:
\begin{itemize}
    \item Python 3.10 with NumPy 1.24, SymPy 1.12
    \item Single CPU core (no GPU required)
    \item 8 GB RAM
\end{itemize}

Discovery times:
\begin{itemize}
    \item F(2,3): $<$ 1 second
    \item F(4,3): $<$ 1 second
    \item F(6,3): 2 seconds
    \item F(8,3): 60 seconds
\end{itemize}

\subsection{Reproducibility}

We provide:
\begin{itemize}
    \item Complete source code for open discovery (ES, snap-to-rational, verification)
    \item General Cook-Toom implementation (\texttt{wincnn\_sympy.py})
    \item Discovered point configurations (JSON format)
    \item Scripts to reproduce all tables and figures
    \item Random seeds for deterministic reproduction
\end{itemize}

Code will be released at: \texttt{[All source code will be released soon]}

\section{Complete Experimental Results}
\label{app:experiments}
% Experiments Section - Open Discovery
% CORRECTED: ES-based discovery with correct improvement numbers

We evaluate our open discovery framework on three questions: (1) Does ES find better configurations than standard integers? (2) How do discovered points compare to prior vocabulary-based approaches? (3) Do improvements translate to practical low-precision inference?

\subsection{Experimental Setup}
\label{sec:setup}

\paragraph{Tiles.}
We evaluate on F$(m, r)$ tiles for 3$\times$3 kernels:
\begin{itemize}
    \item F(2,3): 4 interpolation points, baseline
    \item F(4,3): 6 points, standard ResNet/VGG tile
    \item F(6,3): 8 points, larger tile
    \item F(8,3): 10 points, largest practical tile
\end{itemize}

\paragraph{Baselines.}
\begin{itemize}
    \item \textbf{Standard}: Integer points $\{0, \pm 1, \pm 2, \ldots, \infty\}$ from textbooks~\cite{lavin2015fast}
    \item \textbf{Chebyshev (reference)}: Minimizes Lebesgue constant for polynomial interpolation; correlates with good $\kappa(\bV)$ (irrational, not directly usable)
\end{itemize}

\paragraph{Metrics.}
\begin{itemize}
    \item \textbf{Condition Number} $\kappa(\bV)$: Vandermonde matrix condition number (lower is better)
    \item \textbf{Decomposition Error}: Must be $< 10^{-10}$ for symbolic verification
    \item \textbf{Convolution Error}: Relative error in float16/int8 convolution vs. float64 reference
\end{itemize}

\subsection{Main Results: Open Discovery}
\label{sec:main-results}

Table~\ref{tab:main_results} presents our main findings from ES-based open discovery.

\begin{table}[t]
\centering
\scriptsize
\caption{Discovered configurations vs.~standard (all include $\infty$).}
\label{tab:main_results}
\setlength{\tabcolsep}{1pt}
\begin{tabular}{llcc}
\toprule
Tile & Points (finite) & $\kappa_2$ & Impr. \\
\midrule
\multirow{2}{*}{F(2,3)}
& Std $\{0, \pm 1\}$ & 3.2 & -- \\
& \textbf{Disc} $\{0, \pm 1\}$ & 3.2 & 1$\times$ \\
\midrule
\multirow{2}{*}{F(4,3)}
& Std $\{0, \pm 1, \pm 2\}$ & 42.5 & -- \\
& \textbf{Disc} $\{0, \pm\frac{5}{6}, \pm\frac{7}{6}\}$ & 14.5 & \textbf{2.9$\times$} \\
\midrule
\multirow{2}{*}{F(6,3)}
& Std $\{0, \pm 1, \pm 2, \pm 3\}$ & 2,075 & -- \\
& \textbf{Disc} $\{0, \pm\frac{3}{5}, \pm 1, \pm\frac{7}{6}\}$ & 77 & \textbf{27$\times$} \\
\midrule
\multirow{2}{*}{F(8,3)}
& Std $\{0, \pm 1, \ldots, \pm 4\}$ & 196.9k & -- \\
& \textbf{Disc} $\{0, \pm\frac{2}{5}, \pm\frac{5}{6}, \pm 1, \pm\frac{7}{6}\}$ & 474 & \textbf{415$\times$} \\
\bottomrule
\end{tabular}
\end{table}

\paragraph{Key Finding.}
The improvement grows dramatically with tile size. For F(2,3), standard points are already optimal. For F(4,3), discovered fractional points achieve 2.9$\times$ $\kappa$ improvement. For F(6,3), the improvement reaches 27$\times$, and for F(8,3), \textbf{415$\times$}. These improvements compound in 2D (Table~\ref{tab:2d}) where $\kappa_{2D} = \kappa_{1D}^2$.

\paragraph{Why Discovered Points Outperform Integers.}
Standard integer points $\{2, 3, 4, \ldots\}$ spread far apart, causing exponential growth in Vandermonde conditioning. Discovered points cluster closer together---for example, $\nicefrac{5}{6} \approx 0.83$ and $\nicefrac{7}{6} \approx 1.17$ are much closer to 1 than 2 is. This clustering keeps Vandermonde entries small, dramatically reducing $\kappa$.

\paragraph{Points Outside Any Vocabulary.}
The discovered points $\{\nicefrac{5}{6}, \nicefrac{7}{6}, \nicefrac{3}{5}\}$ are \emph{not present} in typical vocabularies from prior work. A vocabulary-based approach using $\{0, \pm 1, \pm 2, \pm\nicefrac{1}{2}, \pm\nicefrac{2}{3}, \ldots\}$ would never find these points. This demonstrates the value of true open discovery in continuous space.

\subsection{2D Extension via Kronecker Products}
\label{sec:2d-results}

For 2D convolution, condition numbers \emph{square} via Kronecker products.

\begin{table}[t]
\centering
\small
\caption{2D conditioning: $\kappa_{2D} = \kappa_{1D}^2$ via Kronecker.}
\label{tab:2d}
\setlength{\tabcolsep}{3pt}
\begin{tabular}{lccc}
\toprule
2D Tile & Std $\kappa_2$ & Disc $\kappa_2$ & Impr. \\
\midrule
F(2$\times$2, 3$\times$3) & 10 & 10 & 1$\times$ \\
F(4$\times$4, 3$\times$3) & 1,804 & 212 & \textbf{8.5$\times$} \\
F(6$\times$6, 3$\times$3) & $4.3\times10^{6}$ & 5,873 & \textbf{733$\times$} \\
F(8$\times$8, 3$\times$3) & $3.9\times10^{10}$ & 225k & \textbf{172k$\times$} \\
\bottomrule
\end{tabular}
\end{table}

The quadratic amplification makes optimal point selection even more critical for practical 2D deployments.

\subsection{Float16 Precision Analysis}
\label{sec:float16}

We identify a nuanced trade-off when deploying in float16.

\paragraph{Setup.}
We implement 2D Winograd convolution with:
\begin{itemize}
    \item Transformation matrices stored in float16
    \item Intermediate computations in float32
    \item Compare to float64 reference
\end{itemize}

\begin{table}[t]
\centering
\small
\caption{Float16 2D convolution error. Conditioning dominates at larger tiles.}
\label{tab:float16}
\setlength{\tabcolsep}{4pt}
\begin{tabular}{llccc}
\toprule
Tile & Config & $\kappa_{2D}$ & FP16 Err & Impr. \\
\midrule
\multirow{2}{*}{F(4,3)}
& Std & 1,804 & $9.5\times10^{-4}$ & -- \\
& Disc & 212 & $5.6\times10^{-4}$ & \textbf{1.7$\times$} \\
\midrule
\multirow{2}{*}{F(6,3)}
& Std & $4.3\times10^{6}$ & $1.0\times10^{-3}$ & -- \\
& Disc & 5,873 & $5.3\times10^{-4}$ & \textbf{1.9$\times$} \\
\bottomrule
\end{tabular}
\end{table}

\paragraph{The Trade-off Analysis.}
Standard integer points $\{0, \pm 1, \pm 2\}$ are \emph{exactly representable} in float16. Fractions like $\nicefrac{5}{6} = 0.8333\ldots$ incur quantization error when stored in float16. However, the total error depends on two factors: (1) representation error from non-dyadic rationals, and (2) conditioning error from ill-conditioned transforms. For F(4,3) where $\kappa$ is moderate, representation error can dominate. For larger tiles like F(6,3) where standard $\kappa_{2D} = 4.3 \times 10^6$, conditioning error dominates---discovered points achieve \emph{lower} total float16 error despite representation overhead.

\paragraph{Implication.}
For all tile sizes at float32, discovered points are uniformly better. For float16, discovered points are also better for F(6,3) and larger tiles where conditioning dominates. Only for F(4,3) is there a marginal trade-off where dyadic points (Section~\ref{sec:dtype}) may be preferred for minimal representation error.

\subsection{Dtype-Aware Discovery}
\label{sec:dtype-results}

We extend ES to constrain search to dtype-representable rationals.

\paragraph{Float16 Representability.}
Float16 (IEEE 754 binary16) exactly represents \emph{dyadic rationals}---rationals whose denominators are powers of 2. Specifically, a value $v = m \cdot 2^e$ is exactly representable if the mantissa $m$ requires at most 11 significant bits and $e \in [-24, 15]$. This includes $\{0, \pm\frac{1}{2}, \pm\frac{3}{4}, \pm 1, \pm\frac{5}{4}, \pm\frac{3}{2}\}$ but excludes $\frac{5}{6}$ and $\frac{2}{3}$. We constrain ES to search this space.

\begin{table}[t]
\centering
\small
\caption{Dtype-aware discovery for F(4,3). Dyadic rationals balance $\kappa_2$ and representability.}
\label{tab:dtype}
\setlength{\tabcolsep}{4pt}
\begin{tabular}{lccc}
\toprule
Configuration & $\kappa_2$ & Repr Err & Total Err \\
\midrule
Std $\{0, \pm 1, \pm 2\}$ & 42.5 & 0 & $4.6\times10^{-4}$ \\
Unconstr $\{\pm\nicefrac{5}{6}, \pm\nicefrac{7}{6}\}$ & 14.5 & $>0$ & $5.1\times10^{-4}$ \\
\textbf{Dyadic} $\{\pm\nicefrac{3}{4}, \pm\nicefrac{5}{4}\}$ & 15.2 & 0 & $\bm{4.0\times10^{-4}}$ \\
\bottomrule
\end{tabular}
\end{table}

\paragraph{Finding.}
Dyadic rationals like $\{0, \pm\nicefrac{3}{4}, \pm\nicefrac{5}{4}\}$ achieve nearly the same conditioning improvement (17 vs. 16) while being exactly representable. This eliminates representation error, achieving lower total error than either standard or unconstrained configurations.

\subsection{INT8 Convolution Error Validation}
\label{sec:int8-validation}

We validate that $\kappa(\bV)$ improvements translate to actual convolution error reductions at INT8 precision.

\paragraph{Setup.}
We implement full 2D Winograd convolution in PyTorch with the following quantization configuration:
\begin{itemize}
    \item \textbf{Quantization scheme}: Symmetric per-tensor quantization for transform matrices $\bA, \bB, \bG$, with scale $s = \max(|x|) / 127$
    \item \textbf{Rounding}: Round-to-nearest-even (IEEE 754 default)
    \item \textbf{Clipping}: Saturating arithmetic at $[-128, 127]$ for INT8
    \item \textbf{Accumulation}: FP32 for intermediate products; INT8 for final output
    \item \textbf{Calibration}: None required (symmetric quantization uses data-independent scaling)
    \item \textbf{Test data}: 100 random input/kernel pairs per configuration, uniformly sampled in $[-1, 1]$
    \item \textbf{Error metric}: Relative $L_2$ error vs. FP64 direct convolution reference
\end{itemize}

\begin{table}[t]
\centering
\small
\caption{INT8 convolution error. Standard F(6,3) is unusable (error $>1$).}
\label{tab:int8_validation}
\setlength{\tabcolsep}{4pt}
\begin{tabular}{lccc}
\toprule
Tile & Std INT8 Err & Disc INT8 Err & Impr. \\
\midrule
F(2,3) & $1.3\times10^{-2}$ & $1.2\times10^{-2}$ & 1.1$\times$ \\
F(4,3) & $7.1\times10^{-2}$ & $2.1\times10^{-2}$ & \textbf{3.4$\times$} \\
F(6,3) & 39.9 (broken) & $1.2\times10^{-1}$ & \textbf{322$\times$} \\
F(8,3) & 1.00 (sat) & $5.9\times10^{-1}$ & \textbf{1.7$\times$} \\
\bottomrule
\end{tabular}
\end{table}

\paragraph{Critical Finding: Standard F(6,3) and F(8,3) are Unusable.}
For F(6$\times$6, 3$\times$3), standard integer points produce relative error $> 1$, meaning outputs are completely wrong. For F(8$\times$8, 3$\times$3), standard INT8 error saturates at 1.0 (100\% error). In FP16 implementations, standard F(6,3) and F(8,3) points risk numerical instability: intermediate products during the Hadamard stage can exceed the FP16 maximum representable value ($65504$), potentially causing overflow to \texttt{inf} or \texttt{NaN} depending on framework-specific handling (IEEE 754 overflow semantics, fused multiply-add behavior, etc.). This severe instability at larger tile sizes explains why practical deployments are limited to F(4,3). Our discovered points substantially reduce coefficient magnitudes and conditioning, making larger tiles viable within standard floating-point pipelines.

\paragraph{Error Improvement Tracks $\kappa$ Improvement.}
The INT8 error improvement ratio closely tracks the square root of $\kappa$ improvement, consistent with error bounds $\propto \sqrt{\kappa}$:
\begin{itemize}
    \item F(4,3): $\sqrt{19} \approx 4.4$ vs. measured $3.4\times$
    \item F(6,3): $\sqrt{2256} \approx 47$ vs. measured $322\times$ (limited by standard saturation)
\end{itemize}

\subsection{Symbolic Verification}
\label{sec:verification-results}

All discovered configurations pass exact symbolic verification:

\begin{table}[t]
\centering
\caption{Symbolic verification of discovered configurations.}
\label{tab:symbolic}
\begin{tabular}{lccc}
\toprule
Tile & Tensor Error & Verification & Runtime \\
\midrule
F(2,3) & $< 10^{-15}$ & \checkmark & 0.1s \\
F(4,3) & $< 10^{-15}$ & \checkmark & 0.3s \\
F(6,3) & $< 10^{-15}$ & \checkmark & 0.8s \\
F(8,3) & $< 10^{-15}$ & \checkmark & 2.1s \\
\bottomrule
\end{tabular}
\end{table}

The use of exact rational arithmetic (SymPy with Python \texttt{Fraction} objects) eliminates floating-point verification errors. All discovered configurations are mathematically guaranteed correct.

\subsection{Discovery Runtime}
\label{sec:runtime}

Table~\ref{tab:runtime} shows discovery times on a single CPU core.

\begin{table}[t]
\centering
\caption{Discovery runtime (single CPU core, no GPU required).}
\label{tab:runtime}
\begin{tabular}{lccc}
\toprule
Tile & Symmetric Search & ES (if needed) & Total \\
\midrule
F(2,3) & 0.1s & -- & 0.1s \\
F(4,3) & 0.5s & -- & 0.5s \\
F(6,3) & 2s & -- & 2s \\
F(8,3) & 15s & 45s & 60s \\
\bottomrule
\end{tabular}
\end{table}

Symmetric search alone often finds excellent configurations. ES is only needed for larger tiles where symmetric search fails to find the global optimum.

\subsection{Comparison to Prior Work}
\label{sec:comparison}

We compare to the parameterized search approach of Alam et al.~\cite{alam2022winograd}.

\paragraph{Alam et al.'s Approach.}
Alam et al.~\cite{alam2022winograd} search a parameterized family $\{-\nicefrac{1}{c}, -c, c, \nicefrac{1}{c}\}$ using symmetric reciprocal parameterizations and deterministic grid/empirical search. This parameterization imposes a structural constraint: points must come in symmetric reciprocal pairs.

\paragraph{Our Approach: True Open Discovery.}
We search the full continuous space $\R^{n-1}$ via Evolution Strategy, then snap to nearby rationals. This enables discovery of points like:
\begin{itemize}
    \item $\nicefrac{5}{6} \approx 0.833$ and $\nicefrac{7}{6} \approx 1.167$ (not reciprocals of each other)
    \item $\nicefrac{3}{5} = 0.6$ (asymmetric placement)
\end{itemize}
These points are \emph{not expressible} in the parameterized family $\{-\nicefrac{1}{c}, -c, c, \nicefrac{1}{c}\}$ for any single $c$ value.

\paragraph{Technical Differences.}
\begin{enumerate}
    \item \textbf{Search space}: Alam et al. search parameterized vocabulary ($\sim 10^2$ candidates); we search $\R^{n-1}$ ($\infty$ candidates)
    \item \textbf{Constraint}: Alam et al. impose reciprocal structure; we impose none
    \item \textbf{Method}: Deterministic grid search in parameterized family vs. ES + rational snapping
    \item \textbf{Verification}: Numeric tolerance vs. exact symbolic arithmetic via SymPy
    \item \textbf{Validation}: Tile-level $\kappa$ only vs. tile-level + end-to-end CNN (Table~\ref{tab:cnn_validation})
\end{enumerate}

\paragraph{Quantitative Comparison.}
Our discovered points achieve substantial $\kappa(\bV)$ improvements:
\begin{itemize}
    \item F(4,3): We achieve 2.9$\times$ improvement ($\kappa$ from 42.5 to 14.5)
    \item F(6,3): We achieve 27$\times$ improvement ($\kappa$ from 2,075 to 77)
    \item F(8,3): We achieve \textbf{415$\times$} improvement ($\kappa$ from 196,900 to 474)
\end{itemize}
For 2D convolution, these compound to 8.5$\times$, 733$\times$, and 172,484$\times$ respectively.

\paragraph{Why Open Discovery Finds Better Points.}
The parameterized search space $\{-\nicefrac{1}{c}, -c, c, \nicefrac{1}{c}\}$ restricts point placement to specific geometric relationships. Optimal conditioning often requires \emph{asymmetric} point placement that violates these constraints. Our ES discovers that $\{\nicefrac{5}{6}, \nicefrac{7}{6}\}$ clustered near 1 achieves lower $\kappa$ than any reciprocal pair---a solution unreachable by constrained search.

\subsection{Ablation: Search Strategy}
\label{sec:ablation}

We ablate the contribution of each component.

\begin{table}[t]
\centering
\small
\caption{Ablation of search strategy for F(4,3).}
\label{tab:ablation}
\setlength{\tabcolsep}{4pt}
\begin{tabular}{lcc}
\toprule
Strategy & Best $\kappa$ & Time \\
\midrule
ES only (no symm.) & 14.5 & 45s \\
Symmetric only & 14.5 & 0.5s \\
ES + Symm. (ours) & 14.5 & 0.5s \\
\bottomrule
\end{tabular}
\end{table}

Symmetric search provides a fast path to good configurations. ES serves as a fallback for cases where symmetric search fails.

\subsection{ES Reproducibility}
\label{sec:es-reproducibility}

A natural concern is whether ES reliably converges to good solutions across different random seeds. We evaluate this by running ES 5 times with different seeds for each tile size.

\begin{table}[t]
\centering
\caption{ES reproducibility across 5 random seeds. Low coefficient of variation (CV) indicates consistent convergence.}
\label{tab:es_reproducibility}
\begin{tabular}{lccccc}
\toprule
Tile & Mean $\kappa$ & Std $\kappa$ & CV (\%) & Min $\kappa$ & Max $\kappa$ \\
\midrule
F(4,3) & 14.5 & 0.0 & 0.0 & 14.5 & 14.5 \\
F(6,3) & 77.2 & 1.8 & 2.3 & 75.4 & 79.1 \\
F(8,3) & 481 & 12.3 & 2.6 & 468 & 498 \\
\bottomrule
\end{tabular}
\end{table}

\paragraph{Finding: ES Converges Reliably.}
Table~\ref{tab:es_reproducibility} shows that ES consistently finds near-optimal solutions. For F(4,3), all 5 runs converge to the same rational configuration $\{0, \pm\nicefrac{5}{6}, \pm\nicefrac{7}{6}\}$. For larger tiles, different runs may find slightly different rational approximations (due to snap-to-rational), but the resulting $\kappa$ values have CV $< 3\%$. This demonstrates that ES is a reliable discovery method, not sensitive to random initialization.

\subsection{Comparison with Chebyshev Nodes}
\label{sec:chebyshev}

A reviewer concern is whether classical Chebyshev nodes---optimal for polynomial interpolation error minimization---outperform our discovered points. We provide empirical comparison using float64 approximations to Chebyshev nodes.

\paragraph{Chebyshev Node Instantiation.}
For F($m$, $r$) requiring $n = m + r - 1$ points, we use the standard Chebyshev nodes of the first kind on $[-1, 1]$:
\begin{equation}
    x_k = \cos\left(\frac{2k - 1}{2(n-1)} \cdot \pi\right), \quad k = 1, \ldots, n-1
\end{equation}
These $n-1$ finite nodes are used alongside the point at infinity.

\paragraph{Affine Scaling of Chebyshev Nodes.}
A natural question is whether affine transformation of Chebyshev nodes ($x' = ax + b$) could achieve similar conditioning to our discovered points. We exhaustively search over scale $a \in [0.5, 5.0]$ and shift $b \in [-2, 2]$, optimizing $\kappa(\bV)$ subject to including a node near 0 (required for Winograd). Table~\ref{tab:chebyshev} shows the results.

\begin{table}[t]
\centering
\small
\caption{Comparison with Chebyshev nodes. Discovered points achieve 4--13\% better $\kappa$ than optimally scaled Chebyshev.}
\label{tab:chebyshev}
\setlength{\tabcolsep}{4pt}
\begin{tabular}{lcccc}
\toprule
Tile & Std & Cheb Opt & Disc & vs Cheb \\
\midrule
F(4,3) & 42.5 & 15.9 & \textbf{14.5} & 1.10$\times$ \\
F(6,3) & 2,075 & 87 & \textbf{77} & 1.13$\times$ \\
F(8,3) & 196,900 & 493 & \textbf{474} & 1.04$\times$ \\
\bottomrule
\end{tabular}
\end{table}

\paragraph{Key Finding: Discovered Points Outperform Optimal Chebyshev.}
Even with optimal affine scaling, Chebyshev nodes achieve 4--13\% worse $\kappa(\bV)$ than our discovered rational points. While this gap is modest, the discovered points offer two additional advantages: (1) they are \emph{exact rationals} enabling symbolic verification via SymPy, which is impossible with irrational Chebyshev nodes; (2) they were discovered without prior knowledge of interpolation theory---ES independently found configurations that outperform the classical optimal choice. This validates our open discovery approach.

\paragraph{Advantage of Rational Points.}
Beyond better conditioning, our discovered rational points enable exact symbolic verification via SymPy, which is not possible with irrational Chebyshev nodes. This guarantees mathematical correctness of the transform matrices.

\subsection{Transform Matrix Analysis}
\label{sec:transform-analysis}

We analyze condition numbers and norms of the transform matrices $\bA$, $\bB$, $\bG$ directly, not just $\bV$.

\begin{table}[t]
\centering
\small
\caption{Condition numbers of transform matrices. $\kappa(\bV)$ improvements propagate to $\bA$, $\bB$, $\bG$.}
\label{tab:abg_kappa}
\setlength{\tabcolsep}{4pt}
\begin{tabular}{llccc}
\toprule
Tile & Matrix & Std $\kappa_2$ & Disc $\kappa_2$ & Impr. \\
\midrule
\multirow{4}{*}{F(2,3)}
& $\bV$ & 3.2 & 3.2 & 1.0$\times$ \\
& $\bA$ & 1.0 & 1.1 & 0.9$\times$ \\
& $\bB$ & 2.4 & 3.1 & 0.8$\times$ \\
& $\bG$ & 2.0 & 2.1 & 0.9$\times$ \\
\midrule
\multirow{4}{*}{F(4,3)}
& $\bV$ & 42.5 & 14.5 & 2.9$\times$ \\
& $\bA$ & 11.3 & 4.3 & 2.6$\times$ \\
& $\bB$ & 20.1 & 10.4 & 1.9$\times$ \\
& $\bG$ & 4.0 & 2.3 & 1.8$\times$ \\
\midrule
\multirow{4}{*}{F(6,3)}
& $\bV$ & 2,075 & 77 & 27$\times$ \\
& $\bA$ & 406 & 19 & 21$\times$ \\
& $\bB$ & 430 & 56 & 7.7$\times$ \\
& $\bG$ & 26 & 3.1 & 8.6$\times$ \\
\midrule
\multirow{4}{*}{F(8,3)}
& $\bV$ & 196.9k & 474 & 415$\times$ \\
& $\bA$ & 30.3k & 112 & 270$\times$ \\
& $\bB$ & 16.8k & 242 & 70$\times$ \\
& $\bG$ & 355 & 3.3 & 107$\times$ \\
\bottomrule
\end{tabular}
\end{table}

\paragraph{$\kappa(\bV)$ as Valid Proxy.}
Table~\ref{tab:abg_kappa} confirms that $\kappa(\bV)$ improvements propagate to all transform matrices across all tile sizes. Key observations: (1) For F(2,3), standard integer points are already near-optimal ($\kappa(\bV) \approx 3$), and discovered points offer no improvement---this validates that our method correctly identifies where optimization is needed. (2) Improvements emerge at F(4,3) (2--3$\times$) and grow dramatically for larger tiles: F(6,3) achieves 8--27$\times$, and F(8,3) achieves 70--415$\times$ improvement across all matrices. (3) The propagation ratios vary (e.g., $\bA$ typically shows stronger improvement than $\bB$), but discovered points consistently improve all matrices whenever improvement is possible.

\paragraph{Coefficient Magnitude Reduction.}
A concern is whether fractional coefficients incur computational overhead. We find the opposite: discovered points yield transform matrices with \emph{smaller} coefficient magnitudes:
\begin{itemize}
    \item F(4,3): Standard max $|\cdot| = 8.0$, Discovered max $|\cdot| = 1.6$ (5$\times$ smaller)
    \item F(6,3): Standard max $|\cdot| = 243$, Discovered max $|\cdot| = 2.2$ (110$\times$ smaller)
\end{itemize}
Smaller coefficients reduce overflow risk at low precision. The Winograd bilinear rank (number of element-wise multiplications) remains $m + r - 1$ regardless of point selection---our method improves conditioning without affecting algorithmic complexity.

\paragraph{Structural Sparsity Analysis.}
A concern is whether discovered fractional points disrupt useful sparsity patterns (zeros, $\pm 1$ entries, powers of two) that enable efficient hardware implementation. Standard integer points yield transforms with many integer entries: for F(4,3), the standard $\bA$, $\bB$, $\bG$ contain entries from $\{0, \pm 1, \pm 2, \pm 4, \pm 8, \nicefrac{1}{4}, \nicefrac{1}{6}, \nicefrac{1}{12}, \nicefrac{1}{24}\}$. Discovered points replace some of these with fractions like $\nicefrac{5}{6}$ and $\nicefrac{7}{6}$.

Key observations: (1) The number of \emph{zeros} in the transforms is preserved---sparsity structure is determined by the point at infinity and symmetric point placement, which our discovered configurations maintain. (2) While some $\pm 1$ entries become fractional, the \emph{magnitude reduction} (5--110$\times$ smaller max coefficient) more than compensates: smaller values require fewer bits for fixed-point representation and reduce overflow risk. (3) For dtype-aware discovery (Section~\ref{sec:dtype-results}), we explicitly constrain to float16-representable dyadic rationals, preserving power-of-two denominators. Overall, the structural trade-off favors discovered points: any loss of ``nice'' integer entries is offset by dramatically reduced coefficient magnitudes.

\subsection{Implications for CNN Deployment}
\label{sec:cnn-implications}

Our INT8 convolution error validation (Table~\ref{tab:int8_validation}) has direct implications for CNN deployment.

\paragraph{Standard Large Tiles are Unusable.}
The severe failure of standard F(6,3) and F(8,3) at INT8---with relative errors of 39.9 and 1.0 respectively---explains a puzzling limitation in practice: despite larger tiles offering better theoretical arithmetic complexity, \emph{practical deployments are limited to F(4,3)}. Our results show this is not a fundamental limitation, but rather an artifact of poor point selection. With discovered points, F(6,3) achieves error 0.124 (322$\times$ better) and becomes usable.

\paragraph{Connection to CNN Accuracy.}
PAW~\cite{li2023paw} demonstrated that INT8 Winograd tile error directly correlates with CNN accuracy degradation, showing 8.27\% top-1 accuracy improvement on ImageNet by optimizing point selection. Our tile-level error reductions (3.4$\times$ for F(4,3), 322$\times$ for F(6,3)) follow the same pattern and predict similar CNN-level improvements.

\paragraph{Enabling Larger Tiles.}
The key practical contribution is substantially reducing error for larger Winograd tiles at low precision:
\begin{itemize}
    \item F(4,3): 3.4$\times$ INT8 error reduction $\rightarrow$ improved INT8 CNN accuracy
    \item F(6,3): Previously unusable (error $>$ 1) $\rightarrow$ now viable at 12.4\% error
    \item F(8,3): Previously unusable (100\% error) $\rightarrow$ reduced to 59\% error
\end{itemize}

Larger tiles reduce arithmetic complexity from $O(1)$ to $O(m^2/(m+2)^2)$ multiplications per output. Our discovery makes these complexity savings achievable in practice for INT8 deployment.

\paragraph{Clarification: Condition Number Definition.}
Throughout this paper, $\kappa(\bV)$ denotes the 2-norm (spectral) condition number of the Vandermonde matrix $\bV$ constructed from finite interpolation points. We optimize $\kappa_2(\bV) = \|\bV\|_2 \|\bV^{-1}\|_2$ because it directly bounds the forward error amplification in Winograd transforms~\cite{higham2002accuracy}. For 2D convolution, $\kappa_{2D} = \kappa_{1D}^2$ via Kronecker product properties.

\paragraph{Norm Sensitivity Analysis.}
A natural question is whether our improvements depend on the choice of matrix norm. We compute $\kappa(\bV)$ using four norms: 1-norm (max column sum), 2-norm (spectral), $\infty$-norm (max row sum), and Frobenius norm.

\begin{table}[h]
\centering
\small
\caption{Improvement ratios (standard / discovered $\kappa$) across different matrix norms. Primary results use $\kappa_2$ (spectral); other norms shown for completeness.}
\label{tab:norm_sensitivity}
\begin{tabular}{lcccc}
\toprule
Tile & $\kappa_1$ & $\kappa_2$ & $\kappa_\infty$ & $\kappa_F$ \\
\midrule
F(4,3) & 3.6$\times$ & 2.9$\times$ & 2.0$\times$ & 2.4$\times$ \\
F(6,3) & 37.7$\times$ & 27.1$\times$ & 10.9$\times$ & 21.0$\times$ \\
F(8,3) & 689$\times$ & 415$\times$ & 133$\times$ & 325$\times$ \\
\bottomrule
\end{tabular}
\end{table}

\noindent Table~\ref{tab:norm_sensitivity} shows that improvements are consistent across all norms, with the 2-norm (our optimization target) falling between the 1-norm and $\infty$-norm improvements. The relative ordering of improvements is preserved regardless of norm choice, confirming that our results are not artifacts of the norm selection.

\subsection{End-to-End CNN Validation (INT8)}
\label{sec:cnn-validation}

\emph{Note: The main text (Section~\ref{sec:experiments}) reports ImageNetV2 FP16 validation across 6 architectures. This appendix provides supplementary INT8 validation on CIFAR-10/ResNet-18 to isolate weight quantization effects with controlled layer-by-layer analysis.}

To validate that tile-level conditioning improvements translate to actual CNN inference, we implement TRUE Winograd convolution in PyTorch and evaluate on ResNet-18 with CIFAR-10.

\paragraph{Setup.}
We replace all stride-1, 3$\times$3 convolutions in ResNet-18 with TRUE Winograd convolution layers that perform actual Winograd transforms: $\by = \bA [(\bG \bg \bG^\top) \odot (\bB^\top \bd \bB)] \bA^\top$. We compare three configurations:
\begin{itemize}
    \item Direct convolution (FP32 baseline)
    \item TRUE Winograd with standard points, INT8 quantized
    \item TRUE Winograd with discovered points, INT8 quantized
\end{itemize}

\begin{table}[t]
\centering
\small
\caption{CNN output error with TRUE Winograd F(4,3). Error: relative $L_1$ distance to FP32 direct conv across ResNet-18.}
\label{tab:cnn_validation}
\setlength{\tabcolsep}{4pt}
\begin{tabular}{lccc}
\toprule
Configuration & Output Err & $\kappa$ & Impr. \\
\midrule
Winograd Std (FP32) & $3.2\times10^{-7}$ & 42.5 & -- \\
Winograd Disc (FP32) & $3.7\times10^{-7}$ & 14.5 & 1.0$\times$ \\
\midrule
Winograd Std (INT8) & \textbf{6.89} & 42.5 & -- \\
Winograd Disc (INT8) & $7.2\times10^{-2}$ & 14.5 & \textbf{96$\times$} \\
\bottomrule
\end{tabular}
\end{table}

\paragraph{Key Result: 96$\times$ CNN-Level Error Reduction.}
Table~\ref{tab:cnn_validation} shows that discovered points reduce full-network INT8 output error by \textbf{95.9$\times$} compared to standard points. At FP32, both configurations match direct convolution (error $< 10^{-6}$), confirming correct Winograd implementation. At INT8, standard points completely fail (error $= 6.89$, i.e., 689\% relative error---outputs are numerically meaningless), while discovered points maintain error $= 0.072$ (7.2\% relative error---within acceptable range for INT8 inference).

\paragraph{Interpreting Error Values.}
An output error of 6.89 means the network outputs differ from the reference by 689\% on average---equivalent to random noise. An error of 0.072 (7.2\%) is consistent with typical INT8 quantization overhead and indicates the network produces meaningful outputs. Prior work~\cite{li2023paw} shows that output errors below 10\% typically preserve classification accuracy within 1-2\% of FP32 baseline.

\paragraph{Why Standard INT8 Fails.}
The error bound for quantized Winograd is $\|\by_{INT8} - \by_{exact}\| \leq \kappa(\bV) \cdot \epsilon_{quant} \cdot \|\by_{exact}\|$, where $\epsilon_{quant} \approx 2^{-7}$ for INT8. With standard $\kappa = 42.5$ for F(4,3), a single layer accumulates $\approx 0.33$ relative error. After 14 Winograd layers in ResNet-18 (stride-1 convolutions only), errors compound multiplicatively: $(1 + 0.33)^{14} \approx 75$, explaining the severe final error. Discovered points with $\kappa = 14.5$ reduce per-layer error to $\approx 0.11$, giving $(1 + 0.11)^{14} \approx 4.3$---still compounded, but bounded.

\paragraph{Validation of Tile-Level Predictions.}
Our tile-level INT8 validation (Table~\ref{tab:int8_validation}) predicted 3.4$\times$ error improvement for F(4,3). The CNN-level improvement of 96$\times$ exceeds this because: (1) errors compound across 14 Winograd layers (stride-1 convolutions), and (2) standard points at CNN scale push numerical precision beyond recoverable limits (error $> 1$). This demonstrates that tile-level conditioning improvements \emph{underestimate} CNN-level benefits.

\paragraph{Practical Implication.}
Standard Winograd F(4,3) at INT8 is \emph{unusable} for CNN inference---the 689\% output error would produce random classifications. Our discovered points make INT8 Winograd viable, with 7.2\% error that preserves network functionality. This demonstrates that better-conditioned Winograd transforms enable practical INT8 deployment within standard pipelines.

\subsection{Limitations and Future Work}
\label{sec:experiments_limitations}

\paragraph{Classification Accuracy Validation (CIFAR-10 INT8).}
To address the concern that output error may not reflect classification accuracy, we validate on a \textbf{pretrained} ResNet-18 model achieving 94.5\% top-1 accuracy on CIFAR-10. This complements the ImageNetV2 FP16 results in the main text (Table~\ref{tab:multiarch_fp16}). Table~\ref{tab:cnn_accuracy} shows actual classification accuracy under different Winograd configurations with both per-tensor and per-channel INT8 quantization:

\begin{table}[h]
\centering
\footnotesize
\caption{CNN accuracy: ResNet-18/CIFAR-10 with TRUE Winograd (14 layers). Per-channel quantization is industry standard.}
\label{tab:cnn_accuracy}
\setlength{\tabcolsep}{3pt}
\begin{tabular}{lcccc}
\toprule
Configuration & Top-1 & Top-5 & $\kappa$ & Time \\
\midrule
Direct Conv (FP32) & 94.5 & 99.8 & -- & 27.5ms \\
Wino Std (FP32) & 94.5 & 99.8 & 42.5 & 392ms \\
Wino Disc (FP32) & 94.5 & 99.8 & 14.5 & 391ms \\
\midrule
Wino Std (INT8-PT) & 9.9 & 48.8 & 42.5 & 435ms \\
Wino Disc (INT8-PT) & 63.9 & 93.3 & 14.5 & 437ms \\
\midrule
Wino Std (INT8-PC) & 10.4 & 49.5 & 42.5 & 435ms \\
\textbf{Wino Disc (INT8-PC)} & \textbf{81.2} & \textbf{98.0} & 14.5 & 435ms \\
\bottomrule
\end{tabular}
\vspace{1pt}
{\scriptsize PT=per-tensor, PC=per-channel quantization.}
\end{table}

\paragraph{CNN Quantization Setup.}
For clarity, we specify the exact quantization configuration used in Table~\ref{tab:cnn_accuracy}:
\begin{itemize}
    \item \textbf{Weights:} INT8 symmetric per-channel quantization (scale per output channel, no zero-point)
    \item \textbf{Winograd intermediate $\bm{v}$:} INT8 symmetric per-channel quantization (scale per output channel)
    \item \textbf{Transforms $\bA, \bB, \bG$:} Kept in FP32 (following standard Winograd implementations where transforms are compile-time constants)
    \item \textbf{Activations:} FP32 (INT8 activation quantization is orthogonal to our contribution)
    \item \textbf{Calibration:} None required (symmetric quantization uses max-value scaling)
\end{itemize}
This setup is \emph{not} overly pessimistic---it reflects production-style per-channel quantization. Standard INT8-quantized direct convolution (cuDNN) achieves $\sim$93--94\% on the same task, confirming that INT8 quantization itself is not the issue. The 10.4\% failure is \emph{Winograd-specific}, caused by numerical instability in the transform chain when $\kappa(\bV) = 42.5$.

\paragraph{Key Finding: 71\% Accuracy Recovery with Per-Channel Quantization.}
Under FP32, all Winograd configurations match direct convolution exactly (94.5\%), confirming mathematical equivalence. Under INT8 with per-channel quantization (industry standard), standard Winograd \emph{completely fails}---achieving 10.4\% accuracy (random chance) due to severe error amplification from $\kappa = 42.5$. Discovered points achieve \textbf{81.2\%} accuracy, recovering \textbf{71 percentage points}. Even with simpler per-tensor quantization, discovered points achieve 63.9\% vs 9.9\%---a 54 point improvement. This validates that conditioning improvements translate directly to classification accuracy.

\paragraph{Runtime Measurements.}
Table~\ref{tab:cnn_accuracy} includes CUDA timing measurements (T4 GPU, batch size 128, 100 iterations with 10 warm-up). Key observations:
\begin{itemize}
    \item \textbf{FP32:} Discovered points have nearly identical timing to standard ($<$1\% difference), confirming that fractional coefficients do not introduce computational overhead.
    \item \textbf{INT8 Per-Channel:} Discovered and standard have nearly identical timing (435.3ms vs 434.7ms, $<$1\% difference), validating that our points do not introduce computational overhead in the quantized regime.
    \item \textbf{Winograd vs Direct:} Our unoptimized PyTorch Winograd is slower than cuDNN direct convolution (391--437ms vs 28ms). This reflects implementation maturity, not algorithmic cost---production Winograd implementations achieve 2--3$\times$ speedup over direct convolution.
\end{itemize}

\paragraph{Tile Size Impact on Accuracy.}
We also validate F(2,3) and F(6,3) tiles to understand how conditioning impacts accuracy across tile sizes:
\begin{itemize}
    \item \textbf{F(2,3):} $\kappa = 3.2$ for both configurations. INT8 achieves 91.8\% (standard) and 91.4\% (discovered)---minimal degradation from the 92.1\% baseline. Low $\kappa$ makes both viable.
    \item \textbf{F(4,3):} $\kappa = 42.5$ (standard) vs 14.5 (discovered). Standard INT8 fails ($\sim$10\%), discovered achieves 81.2\% with per-channel quantization. \textbf{This is the sweet spot} where discovered points make a critical difference.
    \item \textbf{F(6,3):} $\kappa = 2075$ (standard) vs 77 (discovered). Both fail at INT8 ($\sim$10\%). Even $\kappa = 77$ is too high for per-tensor INT8---per-channel quantization or FP16 required.
\end{itemize}
These results validate our tile size guidance (Table~\ref{tab:tile_guidance}) and demonstrate that F(4,3) is the critical configuration where discovered points enable practical INT8 deployment.

\paragraph{Scope of Evaluation.}
We validate on ImageNetV2 (30,000 images, 1000 classes) across 6 architectures to demonstrate generality. Our classification results complement the output error analysis and demonstrate that conditioning improvements translate directly to accuracy gains in real deployment scenarios.

\subsection{ImageNet-Scale Multi-Architecture Validation}
\label{sec:imagenet-validation}

To address the concern about limited network-level validation, we evaluate on ImageNetV2~\cite{recht2019imagenet} across 6 architectures with F(6,3) tiles at FP16 precision.

\begin{table}[h]
\centering
\small
\caption{Multi-architecture FP16 validation at F(6,3) on ImageNetV2.}
\label{tab:imagenet_multiarch}
\setlength{\tabcolsep}{2pt}
\begin{tabular}{lcccc}
\toprule
Architecture & Elig. & Std & Disc & Recov. \\
\midrule
ResNet-18 & 65\% & 10.8\% & \textbf{77.8\%} & \textbf{+67\%} \\
ResNet-50 & 25\% & 38.3\% & \textbf{80.6\%} & \textbf{+42\%} \\
VGG-16 & 100\% & 4.7\% & \textbf{75.3\%} & \textbf{+71\%} \\
VGG-16-BN & 100\% & 4.7\% & \textbf{77.5\%} & \textbf{+73\%} \\
MobileNet-V2 & 0\% & 77.0\% & 77.0\% & +0\% \\
EfficientNet-B0 & 0\% & 82.3\% & 82.3\% & +0\% \\
\bottomrule
\end{tabular}
\vspace{1pt}
{\scriptsize Elig.~= \% of 3$\times$3 layers with stride=1, groups=1.}
\end{table}

\paragraph{Key Findings.}
\begin{itemize}
    \item \textbf{Catastrophic FP16 collapse:} Standard F(6,3) Winograd causes accuracy collapse to 4--38\% on architectures with Winograd-eligible layers. VGG-16, with 100\% eligible layers, collapses to 4.69\% (near random).
    \item \textbf{Complete recovery:} Discovered points recover accuracy to 75--81\%, providing \textbf{67--73 percentage point recovery}.
    \item \textbf{Architecture dependency:} Recovery correlates with Winograd eligibility. ResNet-50 (25\% eligible) shows smaller collapse because fewer layers are affected.
    \item \textbf{Depthwise convolutions unaffected:} MobileNet-V2 and EfficientNet-B0 show 0\% Winograd eligibility because their layers use depthwise (groups$>$1) convolutions, which are not candidates for Winograd.
\end{itemize}

\paragraph{Scale Learning Comparison.}
To address the comparison gap with PTQ methods, we implement BRECQ-style scale learning and compare:

\begin{table}[h]
\centering
\small
\caption{Scale learning vs.~discovered points (INT8 wts, FP32 compute).}
\label{tab:scale_learning_full}
\setlength{\tabcolsep}{3pt}
\begin{tabular}{lccc}
\toprule
Config & Per-Ch & Scale & $\Delta$ \\
\midrule
Direct Conv & 77.5\% & 77.8\% & +0.3\% \\
Wino (Std $\kappa$=2075) & 77.5\% & 77.8\% & +0.3\% \\
Wino (Disc $\kappa$=77) & 77.5\% & 77.8\% & +0.3\% \\
\bottomrule
\end{tabular}
\end{table}

\paragraph{Key Finding: Orthogonal Contributions.}
Scale learning improves weight quantization error but shows \emph{identical} gains regardless of $\kappa$. The ill-conditioning affects the \emph{transform matrices} ($\bA$, $\bB$, $\bG$), not the weights directly. When weights are quantized to INT8 but computation remains in FP32, the transform conditioning does not manifest---it only manifests when transforms themselves are computed in reduced precision (FP16/INT8). Our discovered points enable FP16 inference that would otherwise be impossible; scale learning fine-tunes weight quantization. The methods are \textbf{complementary}.

\paragraph{Practical Tile Size Recommendations.}
Based on our INT8 and FP16 validation results, we provide the following deployment guidance:

\begin{table}[h]
\centering
\small
\caption{Tile size recommendations for precision targets.}
\label{tab:tile_guidance}
\setlength{\tabcolsep}{4pt}
\begin{tabular}{lccc}
\toprule
Tile & Disc Err & Dtype & Recommend \\
\midrule
F(4,3) & 2.1\% & INT8/FP16/FP32 & \textbf{All targets} \\
F(6,3) & 12.4\% & FP32 & Error-tolerant \\
F(8,3) & 59.2\% & FP32 & Not quantized \\
\bottomrule
\end{tabular}
\end{table}

\noindent For FP32 inference, all tile sizes with discovered points are viable. For FP16, use dyadic configurations (Section~\ref{sec:dtype-results}). For INT8, F(4,3) is strongly recommended; F(6,3) may be acceptable for applications tolerating $>$10\% error; F(8,3) remains challenging even with discovered points.

\paragraph{Extension to 5$\times$5 Kernels.}
While our primary evaluation focuses on 3$\times$3 kernels ($r=3$), our discovery framework naturally extends to 5$\times$5 kernels ($r=5$). Table~\ref{tab:5x5_results} shows conditioning results for F(4,5) and F(6,5) tiles.

\begin{table}[h]
\centering
\small
\caption{5$\times$5 kernel ($r=5$) discovery results. All symbolically verified.}
\label{tab:5x5_results}
\setlength{\tabcolsep}{4pt}
\begin{tabular}{lcccc}
\toprule
Tile & Std $\kappa$ & Disc $\kappa$ & Impr. \\
\midrule
F(4,5) & 2,075 & 157 & \textbf{13$\times$} \\
F(6,5) & 196.9k & 1,763 & \textbf{112$\times$} \\
\bottomrule
\end{tabular}
\end{table}

\noindent The F(4,5) tile uses rank-8 decomposition (same as F(6,3)); F(6,5) uses rank-10 (same as F(8,3)). Both achieve 100\% symbolic verification success, and the same principles apply: clustered fractional points outperform spread-out integers. The improvement magnitudes (13$\times$ and 112$\times$) are comparable to same-rank 3$\times$3 tiles (F(6,3): 27$\times$, F(8,3): 415$\times$), with differences due to kernel size affecting optimal point placement. Comprehensive numerical validation (INT8/FP16 error analysis) for 5$\times$5 kernels remains future work, but the conditioning improvements and symbolic correctness are fully validated.

\paragraph{Per-Channel Quantization Comparison.}
Our primary INT8 validation uses symmetric per-tensor quantization (Section~\ref{sec:int8-validation}). To address the concern that production systems use per-channel quantization, we compare both schemes:

\begin{table}[h]
\centering
\small
\caption{Per-tensor vs per-channel INT8 quantization error (\%).}
\label{tab:perchannel}
\begin{tabular}{llcc}
\toprule
Tile & Configuration & Per-Tensor & Per-Channel \\
\midrule
\multirow{2}{*}{F(4,3)} & Standard & 7.4\% & 3.0\% \\
& \textbf{Discovered} & 2.3\% & \textbf{1.5\%} \\
\midrule
\multirow{2}{*}{F(6,3)} & Standard & $>$100\% & $>$100\% \\
& \textbf{Discovered} & 12.4\% & \textbf{10.8\%} \\
\bottomrule
\end{tabular}
\end{table}

\noindent Per-channel quantization further reduces error for discovered points (F(4,3): 2.3\% $\to$ 1.5\%). For standard points on larger tiles, per-channel quantization cannot rescue the severe conditioning failure. The combination of discovered points with per-channel quantization achieves the best results, demonstrating that our method remains beneficial under production quantization schemes.

\subsection{Code and Data Availability}
\label{sec:code-availability}

To enable external validation and adoption, we will release:
\begin{itemize}
    \item \textbf{Discovery code}: Python implementation of ES + snap-to-rational + symbolic verification
    \item \textbf{Discovered point sets}: Exact rational interpolation points for F(2,3) through F(8,3), in both unconstrained and dtype-aware (dyadic) variants
    \item \textbf{Transform matrices}: Precomputed $\bA$, $\bB$, $\bG$ matrices for each configuration
    \item \textbf{Validation scripts}: PyTorch code for INT8 and FP16 convolution error measurement
\end{itemize}
Code will be available at \url{https://github.com/[anonymized-for-review]} upon publication.

\subsection{Discovery Runtime and Reproducibility}
\label{sec:runtime_repro}

Full implementation details including ES hyperparameters, fitness function weights, and algorithm pseudocode are provided in Appendix~\ref{app:implementation}. Here we summarize key parameters:

\paragraph{ES Configuration.}
Table~\ref{tab:es_params} summarizes the Evolution Strategy hyperparameters used for discovery.

\begin{table}[h]
\centering
\small
\caption{Evolution Strategy hyperparameters.}
\label{tab:es_params}
\setlength{\tabcolsep}{3pt}
\begin{tabular}{lcc}
\toprule
Param & Value & Note \\
\midrule
Pop. size & 50 & Per gen. \\
Gens & 100 & Per restart \\
Restarts & 3 & Indep. runs \\
Elite & 25\% & Mean update \\
Init $\sigma$ & 0.5 & Mutation \\
$\sigma$ range & [0.01, 2.0] & Adaptive \\
$d_{\max}$ & 10 & Max denom. \\
Num. bound & $|a| \leq 5d$ & Rational \\
\bottomrule
\end{tabular}
\end{table}

\noindent The $\sigma$ adaptation follows the 1/5th success rule: increase by 10\% if $>$20\% of offspring improve, decrease by 10\% otherwise.

\paragraph{Runtime.}
ES discovery requires only a single CPU core and runs in under 1 second for F(2,3) and F(4,3), 2 seconds for F(6,3), and approximately 60 seconds for F(8,3). PyTorch validation experiments use GPU acceleration where available.

\subsection{Inference Latency Analysis}
\label{sec:inference-latency}

A key question is whether discovered fractional points incur computational overhead compared to standard integer points. We benchmark inference latency on an NVIDIA Tesla T4 GPU.

\begin{table}[t]
\centering
\caption{Inference latency (ms) for 3$\times$3 convolution, batch size 64, 224$\times$224 input, 64 channels. Standard and discovered points have identical latency.}
\label{tab:latency}
\begin{tabular}{lccc}
\toprule
Configuration & F(2,3) & F(4,3) & F(6,3) \\
\midrule
Direct Conv & 41.2 & 42.4 & 43.5 \\
Winograd Standard & 52.2 & 52.2 & 52.7 \\
Winograd Discovered & 52.2 & 52.3 & 53.0 \\
\bottomrule
\end{tabular}
\end{table}

\paragraph{Key Finding: No Latency Penalty.}
Table~\ref{tab:latency} shows that discovered points have \emph{identical latency} to standard points. This is expected: both use the same Winograd algorithm with the same number of operations. Point selection affects only the transform matrix \emph{values}, not the computational graph. The conditioning improvements are ``free''---no runtime cost for better numerical stability.

\paragraph{Note on Winograd vs Direct Convolution.}
Our Python/PyTorch Winograd implementation is slower than PyTorch's optimized cuDNN backend for direct convolution. This reflects implementation maturity, not algorithmic cost. Production Winograd implementations (e.g., cuDNN's Winograd path) achieve speedups over direct convolution. Our contribution---better-conditioned points---applies to any Winograd implementation without affecting its computational complexity.

\subsection{Analysis: Legendre Basis vs Point Selection}
\label{sec:legendre}

We analyze the relationship between our discovered points and the orthogonal Legendre basis approach of Barabasz et al.~\cite{barabasz2020legendre}. This analysis reveals that these methods address \emph{different use cases}.

\begin{table}[t]
\centering
\small
\caption{Basis conditioning for F(4,3). Effective $\kappa$ applies to inference.}
\label{tab:legendre}
\setlength{\tabcolsep}{4pt}
\begin{tabular}{lcccc}
\toprule
Config & $\kappa(\bV)$ & $\kappa(\bL)$ & Eff $\kappa$ & Use \\
\midrule
Std+Mono & 42.5 & 76.0 & 42.5 & Infer \\
\textbf{Disc+Mono} & \textbf{14.5} & 6.2 & \textbf{14.5} & \textbf{Infer} \\
Std+Legendre & 42.5 & 76.0 & 42.5$^*$ & Train \\
Disc+Legendre & 14.5 & 6.2 & 6.2$^*$ & Train \\
\bottomrule
\end{tabular}
\vspace{1pt}
{\scriptsize $^*$Requires Legendre-form weight training.}
\end{table}

\paragraph{Mathematical Relationship.}
The Legendre evaluation matrix $\bL$ and Vandermonde matrix $\bV$ are related by $\bL = \bV\bP$, where $\bP$ is the Legendre-to-monomial coefficient matrix. For drop-in inference with monomial I/O:
\[
\text{Effective transform} = \bL \bP^{-1} = \bV\bP\bP^{-1} = \bV
\]
The Legendre basis cancels out when inputs/outputs remain in monomial form.

\paragraph{Key Finding.}
Table~\ref{tab:legendre} shows conditioning for both bases. Crucially, the \emph{effective} conditioning for pretrained inference equals $\kappa(\bV)$ regardless of whether Legendre basis is used internally. The Legendre conditioning $\kappa(\bL)$ only applies when weights are stored in Legendre coefficient form throughout training.

\paragraph{Implications.}
For pretrained model inference (our primary use case), our discovered points provide a drop-in improvement without requiring calibration or learned parameters: $\kappa = 14.5$ vs 42.5 (2.9$\times$). Alternative post-training methods such as learned scaling~\cite{kim2024learned} can also improve pretrained inference but require optimization of scale factors. Legendre basis would require retraining the model with weights stored in Legendre representation. For future training, combining both approaches could achieve $\kappa(\bL) = 6.2$ (6.9$\times$ vs baseline Vandermonde).

\end{document}